\documentclass[nohyperref]{article}

\usepackage{microtype}
\usepackage{graphicx}
\usepackage{subfigure}
\usepackage{booktabs} 

\usepackage{hyperref}

\usepackage[accepted]{icml2023}

\usepackage{amsmath}
\usepackage{amssymb}
\usepackage{mathtools}
\usepackage{amsthm}
\usepackage{nicefrac}
\usepackage{bbm}

\usepackage[capitalize,noabbrev]{cleveref}

\usepackage{enumitem,nicefrac}

\theoremstyle{plain}
\newtheorem{theorem}{Theorem}[section]

\theoremstyle{definition}

\newtheorem{assumption}[theorem]{Assumption}
\theoremstyle{remark}

\newcommand{\1}{\mathbbm{1}}

\usepackage[textsize=tiny]{todonotes}

\icmltitlerunning{Iterative Deepening Hyperband}

\usepackage{placeins}
\usepackage{longtable}
\usepackage{tikz}
\usepackage{pgfplots}
\usepackage{xspace}

\newcommand{\idhb}{ID-HB\xspace}
\newcommand{\discardingidhb}{dID-HB\xspace}
\newcommand{\conservativehb}{pID-HB\xspace}
\newcommand{\efficienthb}{eID-HB\xspace}

\newcommand{\hb}{\texttt{Hyperband}\xspace}

\usepackage{eqparbox}
\renewcommand{\algorithmiccomment}[1]{\hfill\eqparbox{COMMENT}{\textcolor{gray}{$\triangleright$ #1}}}

\begin{document}

\twocolumn[
\icmltitle{Iterative Deepening Hyperband}



\icmlsetsymbol{equal}{*}

\begin{icmlauthorlist}
\icmlauthor{Jasmin Brandt}{equal,upb}
\icmlauthor{Marcel Wever}{equal,mcml}
\icmlauthor{Dimitrios Iliadis}{gent}
\icmlauthor{Viktor Bengs}{lmu}
\icmlauthor{Eyke H{\"u}llermeier}{mcml}
\end{icmlauthorlist}

\icmlaffiliation{mcml}{MCML, LMU Munich, Munich, Germany}
\icmlaffiliation{lmu}{LMU Munich, Munich, Germany}
\icmlaffiliation{upb}{Paderborn University, Paderborn, Germany}
\icmlaffiliation{gent}{Ghent University, Ghent, Belgium}

\icmlcorrespondingauthor{Marcel Wever}{marcel.wever@ifi.lmu.de}

\icmlkeywords{Machine Learning, ICML}

\vskip 0.3in
]



\printAffiliationsAndNotice{\icmlEqualContribution} 

\begin{abstract}
Hyperparameter optimization (HPO) is concerned with the automated search for the most appropriate hyperparameter configuration (HPC) of a parameterized machine learning algorithm.  
A state-of-the-art HPO method is \texttt{Hyperband}, which, however, has its own parameters that influence its performance. 
%
One of these parameters, the maximal budget, is especially problematic: If chosen too small, the budget needs to be increased in hindsight and, as \texttt{Hyperband} is not incremental by design, the entire algorithm must be re-run. This is not only costly but also comes with a loss of valuable knowledge already accumulated.    
In this paper, we propose incremental variants of \texttt{Hyperband} that eliminate these drawbacks, and show that these variants satisfy theoretical guarantees qualitatively similar to those for the original \texttt{Hyperband} with the ``right'' budget.
Moreover, we demonstrate their practical utility in experiments with benchmark data sets.
\end{abstract}

\section{Introduction}
The successful application of a machine learning (ML) algorithm to a particular learning problem is usually accompanied by the question of how to choose its hyperparameters.
These parameters are usually highly configurable and, in addition, often have a substantial impact on important meta-properties of the learner such as its complexity, its learning speed, or its (final) performance. 
In light of this, careful selection of these parameters is critical to assessing a learner's suitability for the learning problem at hand.
However, manually searching for good or even optimal hyperparameters is a time-consuming and costly endeavor that may even introduce human bias or prevent reproducibility. 
What complicates matters is that the optimal hyperparameter choice may vary from task to task, i.e., may depend on the learning problem, the nature of the data set, and the corresponding performance measure.

Due to the pressing need to automate this process both in research and in practice, the research field of hyperparameter optimization (HPO) has emerged \citep{feurer2019hyperparameter,bischl2021hyperparameter}. 
In general, methods for HPO problems try to find optimal hyperparameter configurations (HPCs) of a given machine learning algorithm as efficiently as possible, such that,  with the hyperparameters found, it generalizes well to new, unseen data points.
A commonly used method for HPO is \texttt{Hyperband} \citep{LiHyperband}, which is based on the multifidelity concept: The available computational budget is first used on low-cost HPCs for exploration and gradually spent for the most promising HPCs. 	
Here, the budget can refer to computational resources such as the number of instances or attributes to use for training, or the number of iterations for iterative algorithms or number of epochs of neural networks.

However, with the maximum budget and the rate of exploration, \texttt{Hyperband} itself has parameters that determine its quality for the HPO problem.
Especially the choice of the former is problematic since in many cases its choice is not obvious from the beginning, and choosing it too high can be a costly affair, as running \texttt{Hyperband} is already costly on its own.
Examples of such cases are when the budget corresponds to a hyperparameter of the ML algorithm that steers how long it will process the data, such as the number of epochs of a neural network, the number of iterations for logistic regression, and so on.
%
Exactly in these cases a too small choice of the budget is fraught with problems, as the HPCs under consideration may not have developed their potential yet, i.e., the models have not been trained until reaching their convergence. 
However, increasing the budget in hindsight leads to a usually expensive re-run of \texttt{Hyperband} and even worse, discarding valuable knowledge already accumulated.  
Needless to say, from an ecological perspective, this is undesired either, as the computational resources, as well as the consumed energy for optimizing the hyperparameters for the lower maximum budget, is essentially wasted \citep{tornede2021towards}.

To overcome these issues, we propose incremental extensions of \texttt{Hyperband} that are capable of increasing the maximum budget post hoc and continuing the earlier \texttt{Hyperband} run for the smaller maximum budget. 
In the spirit of heuristic graph search algorithms such as iterative deepening A* or iterative deepening search \citep{KORF198597}, we dub our extension \texttt{Iterative Deepening Hyperband} (ID-HB), which comes in three variants.
Each of these variants considers a different possibility to reuse the information gathered by the previous \texttt{Hyperband}, which essentially differ in the degree of conservatism.
We provide theoretical and empirical results showing that the performance deterioration is negligible compared to the improvements in terms of efficiency. 
To this end, we provide theoretical guarantees as well as empirical evaluations for various tasks on benchmark datasets.

\newcommand{\cX}{\mathcal{X}}
\newcommand{\cY}{\mathcal{Y}}
\newcommand{\cH}{\mathcal{H}}
\renewcommand{\vec}[1]{\boldsymbol{#1}}
\section{Hyperparameter Optimization}
In a typical supervised ML setting, the learner is provided with a (training) data set $\mathcal{D}= \big\{ \big(\vec{x}^{(n)} , y^{(n)} \big) \big\}_{n=1}^N \subset \cX \times \cY $, where $\cX$ is some feature space and $\cY$ is a label space.
The pairs in  $\mathcal{D}$ are assumed to be i.i.d.\ samples of some unknown data-generating distribution $P^*,$ i.e., each $z^{(n)}= (\vec{x}^{(n)} , y^{(n)} )$ is an independent realization of $Z=(X,Y) \sim P^*$.
In addition, the learner is provided with a (suitable) hypothesis space $\cH$, which is a subset of all mappings $\cY^{\cX} = \{ h:\cX \to \cY  \}$ from the feature space $\cX$ to the label space $\cY$.
%
Thus, each hypothesis $h \in \cH$ assigns a prediction $h(\vec{x}) \in \cY$ to a provided instance $\vec{x}\in \cX.$
The prediction quality of a hypothesis for a given instance-label pair $(\vec{x},y)$ is assessed by means of $L(h(\vec{x}),y),$ where $L:\cY \times \cY \to \mathbb R$ is a loss function that incentives correct predictions. 
The goal in supervised ML is to induce a hypothesis that has the lowest expected loss (generalization error) for instance-label pairs $(\vec{x},y)$ sampled according to $P^*.$
Formally, the learner seeks to find
$$	h^* \in \arg\min_{h\in \cH} \, \mathrm{GE}(h) \, , $$
where $\mathrm{GE}(h)$ denotes the generalization error of $h$:
$$
\mathrm{GE}(h) =  \mathbb{E}_{(\vec{x},y) \sim P^*}	[ L(h(\vec{x}),y)] \, .
$$
To this end, a learner (or inducer) $\mathcal{A}$ is designed that returns for a given (training) data set a hypothesis deemed to be suitable (having low generalization error) for the learning task at hand.
Typically, a learner is parameterized by a parameter space $\Lambda,$ whose elements are called hyperparameters, which may be high-dimensional tuples with components from different domains (continuous, discrete, or categorical).
Thus, a learner is formally a mapping
\begin{align*}
	\mathcal{A}:\  \mathbb{D} \times \Lambda  \to \cH, \, (\mathcal{D},\lambda) \mapsto \hat h \, ,
\end{align*}
where $\mathbb{D} = \bigcup_{N \in \mathbb{N}} (\cX \times \cY)^N$ is the set of all possible training data sets.
It is worth noting that the learner can be defined in a similar way, if the hypothesis space $\cH$ is a subset of all mappings from the feature space $\cX$ to $\mathbb{P}(\cY),$ i.e., the set of probabilities over the label space $\cY.$
The only difference is that the loss function has a different signature in this case, namely it is a mapping $L:\mathbb{P}(\cY) \times \cY \to \mathbb{R}.$

The goal of hyperparameter optimization (HPO) is then to find an optimal hyperparameter for a given learner $\mathcal{A}$ and data set $\mathcal{D},$ i.e.,  
$$  \lambda^* \in \arg\min_{\lambda \in \Lambda} \, \ell(\lambda) = \arg\min_{\lambda \in \Lambda} \, \mathrm{GE} \big( \mathcal{A}(\mathcal{D},\lambda) \big). $$
However, as $P^*$ is unknown and so is the generalization error $\ell,$ one estimates the latter for a fixed hyperparameter by means of a function $\widehat{\ell}:\Lambda \to \mathbb{R},$ which is typically referred to as the validation error.
As the actual computation of the validation error for a specific hyperparameter might be costly in terms of the available resources (e.g., wall-clock time, number of used data points, etc.), the validation error is usually determined only for a certain resource allocation, and thus its actual value depends on the resources used.
In light of this, we denote by $\widehat{\ell}_r(\lambda)$ the validation error of $\mathcal{A}$ used with the hyperparameter $\lambda$ and $r$ resource units. Obviously, the choice of $r$ involves a trade-off: The more resource units are used, the more accurate the estimate, but the more costly its calculation, and vice versa. 

Roughly speaking, an HPO method seeks to find an appropriate hyperparameter of $\mathcal{A}$, while preferably using as few resources as possible, and/or staying within a maximum assignable budget $R$ for the resource consumption for evaluating a hyperparameter during the search.
For the sake of convenience, we assume that $R$ is an element of $\mathbb{N} \cup \{\infty\},$ where $R=\infty$ means that there are no restrictions on the resource usage.
We define $\ell_*(\lambda) := \lim_{r \rightarrow R} \ell_r(\lambda)$ for any $\lambda \in \Lambda$ and $\nu_* := \inf_{\lambda \in \Lambda} \ell_* (\lambda)$. 
The formal goal of an HPO method is then to identify a hyperparameter $\lambda$ which belongs to $\arg\min_{\lambda \in \Lambda} \ell_*(\lambda) - \nu_*.$

%
        %

\section{Hyperband}
In this section, we explain shortly the functionality of the \hb algorithm by \citet{LiHyperband} and its subroutine Successive Halving \citep{karnin2013almost}.
\subsection{Successive Halving}
The Successive Halving (SH) algorithm solves the non-stochastic best arm identification problem within a fixed budget and was already applied successfully to HPO. It iteratively allocates the available budget to a set of hyperparameter configurations, evaluates their performance and throws away the worst half. This process is repeated until only one hyperparameter configuration remains, which is then returned by the algorithm as the proposed winner. Due to this, we can allocate exponentially more budget to more promising hyperparameter configurations.
\subsection{Hyperband}
The \texttt{Hyperband} algorithm by \citet{LiHyperband} solves the HPO problem by considering it as a pure-exploration adaptive resource allocation problem. It iterates over different sizes of the set of hyperparameter configurations $n$ while keeping the budget $B$ fixed, and calls the SH algorithm as a subroutine on the $n$ configurations and with budget $B$. This way, different allocation strategies are considered for the tradeoff between (i) considering many configurations $n$ and (ii) giving the configurations longer training time $B/n$. Each call of SH is called a \emph{bracket}.

\section{Related Work}
To achieve state-of-the-art performance, hyperparameter optimization (HPO) is an inevitable step in the machine learning process, dealing with finding the most suitable hyperparameter configuration of a machine learning algorithm for a given dataset and performance measures. Considering HPO as a black-box optimization problem, various methods can be used to tackle this problem. However, one particular challenge in HPO is that evaluating a hyperparameter configuration is expensive, rendering naive approaches such as grid search and random search, although widely applied, impractical.

The HPO literature can be separated into two branches: model-free and model-based methods. While the latter leverage a surrogate model of the optimization surface to sample more promising candidates \cite{DBLP:conf/lion/HutterHL11}, for instance, methods evolutionary algorithms fall into the former category of model-free methods. Notably, grid search and random search belong to the model-free category too. While standard random search is considered to be too inefficient in the HPO setting, in \texttt{Hyperband} \cite{LiHyperband}, a random search is combined with a multi-fidelity candidate evaluation routine, i.e., successive halving \cite{JamiesonSuccHalv}, devising a powerful HPO method. Moreover, the model-free approach \texttt{Hyperband} can be combined with other model-free methods such as evolutionary algorithms \cite{DBLP:conf/ijcai/AwadMH21} or hybridized with model-based approaches \cite{DBLP:conf/icml/FalknerKH18} to improve its efficacy and efficiency. In \cite{mendes2021hyperjump}, the authors propose a meta-learning approach to focus the budget for evaluation even more on promising candidates instead of wasting the budget on hyperparameter configurations performing inferior to the incumbent.

While these approaches aim at increasing \texttt{Hyperband}'s efficacy and efficiency for a single run of \texttt{Hyperband}, in this paper, we are interested in increasing \texttt{Hyperband}'s efficiency after increasing its budget that can be assigned to a single hyperparameter configuration at maximum. The general setting was already considered and analyzed in \cite{LiHyperband} as the infinite horizon setting, where \texttt{Hyperband} is run repeatedly for increasing maximum budgets $R$. We continue a previously conducted \texttt{Hyperband} instead of starting from scratch every time the maximum budget is increased.

For a more detailed introduction to HPO and a more thorough overview of corresponding methods, we refer the reader to \cite{feurer2019hyperparameter,bischl2021hyperparameter}.

\section{Iterative Deepening Hyperband}\label{sec:id-hb}

In the infinite horizon setting of \hb, once a run of \hb terminates, a new run is started from scratch with an increased maximum assignable budget $R$. However, the previously evaluated hyperparameter configurations are discarded completely for the \hb run with the increased $R$, which means that the previously used budget is essentially wasted. Discarding this information is irrational due to various reasons:
\begin{itemize}
		[noitemsep,topsep=0pt,leftmargin=4mm]
    \item We are ignoring knowledge about promising hyperparameter configurations that we already acquired.
    \item Instead of using the information we already collected, we evaluate new candidates, allowing for more exploration, which is actually desired but also requires more budget.
    \item From an ecological perspective, the resources in the previous \hb runs are not well invested, since the information, if at all, has been solely used for deciding to re-run \texttt{Hyperband} for a larger budget.
\end{itemize}

To leverage already collected information, we propose an extension to \hb, which is stateful and thus can be continued for a larger maximum budget $R$, which we dub Iterative Deepening Hyperband (\idhb). The name for our method is inspired by iterative deepening search \cite{KORF198597}, e.g., iterative deepening depth-first search, where the depth-first search is invoked repeatedly with increasing maximum depth. The pseudocode for this extension is given in Algorithm~\ref{alg:iterative-deepening-hb} and Figure \ref{fig:IllustrationIDHB} illustrates the core idea.

From Algorithm~\ref{alg:iterative-deepening-hb}, we can see that the old max size $R_{t-1}$ is increased by a factor $\eta$ to obtain $R_t$.
For sake of convenience w.r.t.\ theoretical analysis, we limit ourselves to increases in the maximum budget by a factor of $\eta$.
According to the new max size $R_t$, the variables $s_\text{max}$ and $B$ are computed as before by \citet{LiHyperband}, which will result in an additional bracket for the new max size $R_t$ as well as increased pool sizes for the already existing brackets.
Therefore, we fill up the brackets with additional hyperparameter configurations until the correct start pool size with respect to $R_t$ is reached.
However, since typically more newly sampled hyperparameter configurations are added than can be propagated to the next iteration of a bracket, different strategies of how to proceed with the state of the previous \hb and the newly sampled hyperparameter configurations are conceivable.

In the subsequent Sections~\ref{sec:did-hb} to \ref{sec:eid-hb}, we elaborate on three possible strategies of how a previous \texttt{Hyperband} for a lower maximum budget $R_{t-1}$ can be continued for a larger maximum budget $R_{t}$.
We order the strategies according to their truthfulness with respect to decisions taken in a \texttt{Hyperband} run that would have been run from scratch. This order also goes hand in hand with improved efficiency, i.e., less budget is accumulated for evaluating hyperparameter configurations.

\subsection{Discarding Iterative Deepening (\discardingidhb)}\label{sec:did-hb}
Arguably, the most truthful but potentially inefficient way to update the brackets with the newly received hyperparameter configurations is to allow for revising previous decisions regarding propagation to subsequent iterations. More specifically, when the start pool of hyperparameter configurations is extended by the new hyperparameter configurations, we allow discarding hyperparameter configurations that were promoted in the previous run and have already been evaluated on a larger budget.

In the first iteration, we consider all the candidates available, i.e., the newly sampled hyperparameter configurations, previously discarded, and previously promoted hyperparameter configurations. The top-k is computed as before, and the selected hyperparameter configurations are promoted to the next iteration. Hence, discarding iterative deepening Hyperband is able to revise its previous decisions regarding promotions and discard hyperparameter configurations that have been promoted in the previous run.

While only those hyperparameter configurations in an iteration need to be evaluated that were newly sampled and eventually the last iteration for the new max size $R_t$, it may happen that only new configurations are promoted to the subsequent iteration. In this case, this variant of the stateful extension will only save the resources already used for evaluating the old configurations for the minimum budget of a bracket. In practice, however, we will see in Section~\ref{sec:empirical-results} that old candidates are often kept in subsequent iterations.

Eventually, we have obtained a variant of Iterative Deepening \texttt{Hyperband} that has the potential of improving substantially in terms of efficiency while maintaining the same outcome if the same set of initial hyperparameter configurations was given to the original \texttt{Hyperband}. In the worst case, however, this variant is almost as expensive as re-running the original version, i.e., only the first iteration evaluations of the old candidates are saved.

\subsection{Preserving Iterative Deepening (\conservativehb)}\label{sec:cid-hb}

Taking a step towards more efficiency and reuse of previous evaluations of hyperparameter configurations, in preserving iterative deepening \hb (pID-HB), we promote the top-k hyperparameter configurations of a pool of candidates comprising the promoted hyperparameter configurations and all hyperparameter configurations that have already been promoted to this budget level in the previous Hyperband run but have not been promoted in the continued run. In this way, we conserve the information about hyperparameter configurations that have already been evaluated for this budget but have been discarded in a previous iteration.

Still considering such hyperparameter configurations allows already discarded hyperparameter configurations to return to the pool of promising candidates. On the one hand, we can thereby potentially increase the efficiency, since after a hyperparameter configuration returns to the set of promoted candidates, we do not need to spend additional budget on evaluating a new candidate, except for the last iteration with the budget of the new max size. On the other hand, we can revise ``wrong'' decisions of the current run if a previously discarded old hyperparameter configuration performs better on a larger budget than the configuration that has superseded it.

In the worst case, pID-HB exhibits the same computational complexity as dID-HB, only saving the computational resources of the old hyperparameter configurations being evaluated on the start budgets of the respective brackets. However, due to the ability to reconsider already discarded candidates that have already been evaluated in a previous \texttt{Hyperband} run, the chances of reusing already evaluated candidates increase. Although pID-HB does not necessarily return the same result as re-running \texttt{Hyperband} from scratch, intuitively, we expect it to yield similar performance. Moreover, although pID-HB has the same worst-case runtime complexity as dID-HB, pID-HB at least takes all available information into account, i.e., evaluation data of previously evaluated but discarded hyperparameter configurations is not ignored but still considered.

\subsection{Efficient Iterative Deepening (\efficienthb)}\label{sec:eid-hb}
The third and last strategy is to continue the previous \texttt{Hyperband} run in the most efficient and thus resource-saving way.
Accordingly, we dub this variant ``efficient iterative deepening \texttt{Hyperband}'' (eID-HB).
However, the maximum efficiency comes at the cost of potential performance deteriorations, as it does not revoke any previously made decisions.

In other words, it only fills up the candidate pools promoting hyperparameter configurations until the new levels are reached.
If a hyperparameter configuration was promoted to a subsequent iteration of successive halving it remains, even if there are hyperparameter configurations in the newly sampled set of hyperparameter configurations that would actually replace them.
Technically, when determining the top-k for the next iteration of successive halving we subtract the number of hyperparameter configurations that have already been promoted in the previous run. According to the difference, new promotions are selected among previously discarded candidates and newly sampled hyperparameter configurations that were promoted to the current iteration.

In principle, it may happen that some hyperparameter configurations may have been wrongly promoted as compared to starting the \texttt{Hyperband} run from scratch. Since already promoted hyperparameter configurations remain promoted and decisions cannot be revised, the overall performance may deteriorate. However, if the deterioration would be negligible, \efficienthb would allow us to run \texttt{Hyperband} for a larger max size $R_t$ at the cost of only running \texttt{Hyperband} for the larger max size $R_t$.

In the subsequent sections, we provide theoretical guarantees and also demonstrate empirically that the proposed extensions improve the efficiency significantly while maintaining similar performance.

\begin{algorithm*}[ht!]
\caption{IterativeDeepening-Hyperband (ID-HB)}\label{alg:iterative-deepening-hb}

\begin{algorithmic}
\STATE \textbf{Inputs: } old max size $R_{t-1}$, $\eta \geq 2$, old losses $L_{(t-1,\cdot,\cdot)}$, discarded configurations $D_{(t-1,\cdot,\cdot)}$, promoted configurations $P_{(t-1,\cdot,\cdot)}$, mode flag $\rho \in \{ p, d, e\}$
\STATE \textbf{Initialize:} $R_{t} \gets \eta R_{t-1}$, $s_{\text{max}} \gets \lfloor \log_{\eta}(R_t) \rfloor$, $B \gets (s_{\text{max}}+1)R_t$
\FOR{$s \in \{ s_{\text{max}}, s_{\text{max}}-1,\ldots,0 \}$}
    \STATE $n_t \gets \lceil \dfrac{B}{R_t} \dfrac{\eta^s}{(s+1)} \rceil$
    \STATE $r = R \eta^{-s}$
    \IF[Configurations for old max size]{$s > 0$}
        \STATE $n_{t-1} \gets \lceil \dfrac{\eta^{s-1} s_\text{max}}{s} \rceil $
    \ELSE 
        \STATE $n_{t-1} \gets 0$
    \ENDIF
    \STATE $\delta \gets n_t - n_{t-1}$
    \STATE $C \gets \text{get\_hyperparameter\_configuration}(\delta)$ 
    \FOR{$i \in \{0,\ldots,s\}$}
        \STATE $r_i \gets r\eta^i$
        \STATE $L_{(t, s, i)} \gets L_{(t-1,s-1,i)} \cup \{ \text{run\_then\_return\_val\_loss}(c, r_i): c \in C \setminus P_{(t-1,s-i,i-1)}\}$
        \STATE \COMMENT{Evaluate configurations}

        \IF[Discarding ID-HB]{$\rho = d$}
            \IF{$i=0$} 
                \STATE $T \gets D_{(t-1,s_1,i)} \cup P_{(t-1,s-1,i)} \cup C$
            \ELSE
                \STATE $T \gets C$
            \ENDIF
            \STATE $k \gets \lfloor \nicefrac{n_t}{\eta^{i+1}} \rfloor$
        \ELSIF[Preserving ID-HB]{$\rho = p$} 
            \STATE $T \gets D_{(t-1,s_1,i)} \cup P_{(t-1,s-1,i)} \cup C$
            \STATE $k \gets \lfloor \nicefrac{n_t}{\eta^{i+1}} \rfloor$
        \ELSIF[Efficient ID-HB]{$\rho = e$}
         \STATE $T \gets D_{(t-1,s-1,i)} \cup C$ 
         \STATE $k \gets \lfloor \nicefrac{n_t}{\eta^{i+1}} \rfloor - \lfloor \nicefrac{n_{t-1}}{\eta^{i+1}} \rfloor$
        \ENDIF

        \STATE $C \gets \text{top}_k(T, L_{(t,s,i)},  k)$
        
        \STATE $D_{(t, s, i)} \gets T \setminus (C \cup P_{(t-1,s-1,i)})$ \COMMENT{Update discarded configurations}
        \STATE $P_{(t, s, i)} \gets P_{(t-1,s-1,i)} \cup C$ \COMMENT{Update promoted configurations}
    \ENDFOR
\ENDFOR

\end{algorithmic}
\end{algorithm*}

\begin{figure*}[h!]
    \label{fig:IllustrationIDHB}
    \centering
    \includegraphics[width=0.8\textwidth]{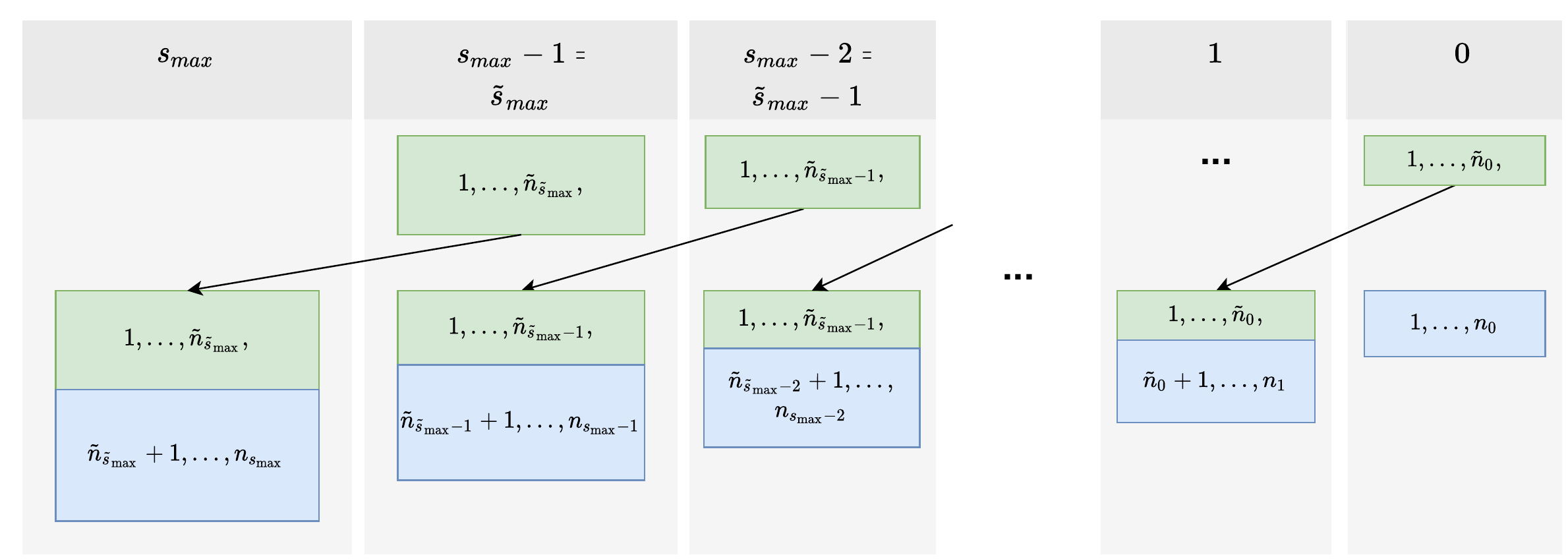}
    \caption{Illustration of taking over previously sampled configurations.}
\end{figure*}

\section{Theoretical Results}
We split the theoretical results into two parts. First, we give some theoretical guarantees for our extensions of SuccessiveHalving, and second, we extend them to the IterativeDeepening-Hyperband algorithm.  
\subsection{IterativeDeepening-SuccessiveHalving}
Since the Successive Halving algorithm \cite{JamiesonSuccHalv} solves a bandit problem, we will stick in the following analysis of our iterative deepening variants of Successive Halving to the notation of multi-armed bandits. Our algorithms can then easily be applied to hyperparameter optimization by regarding a hyperparameter configuration as an arm. If we pull an arm $i$ for $k$ times, we observe the loss $\ell_{i,k}$.
Similar to \citet{LiHyperband} we need the following assumption for our theoretical analysis of our proposed IterativeDeepening-SuccessiveHalving algorithms.
\begin{assumption}
    For each arm $i \in \mathbb{N}$ the limit $\nu_{i} := \lim_{t \rightarrow \infty} \ell_{i,t}$ exists.
\end{assumption}
Moreover, we denote the convergence speed by $\gamma(j) \geq \sup_i |\ell_{i,j} - \nu_i|$ $ \forall j \in \mathbb{N}$ and provide the following result, the proof of which is given in \cref{subsec:Proof_necessary_budget}. 
\begin{theorem}[Necessary Budget for IterativeDeepening-SuccessiveHalving] \label{thm:Budget_ID-SH}
    Fix $n$ arms from which $\tilde{n}$ arms were already promoted. Let $\nu_i = \lim_{t\rightarrow \infty}\ell_{i,t}$ and assume $\nu_1 \leq \dots \leq \nu_n$. For any $\epsilon >0$ let 
    \begin{align*}
        z_{ID-SH} =  &\eta \lceil\log_\eta(n)\rceil\\
        \times \max_{i=2,\dots,n} & i\Big(1 + \min\big \{R, \gamma^{-1} \left( \max\left\{ \tfrac{\epsilon}{4} , \tfrac{\nu_{ i }- \nu_{1}}{2} \right\} \right) \big\}  \Big).
    \end{align*} 
    If the efficient, discarding or preserving IterativeDeepening-SuccessiveHalving algorithm given in Algorithm \ref{alg:eID-SH}, Algorithm \ref{alg:dID-SH} resp. \ref{alg:cID-SH} are run with any budget $B \geq z_{ID-SH}$, then an arm $\hat{\imath}$ is returned that satisfies $\nu_{\hat{\imath}} - \nu_1 \leq \epsilon/2$. 
\end{theorem}
Further, we can specify the improvement of the incremental variants over the costly re-run of SuccessiveHalving (SH) as in the following theorem (proof is deferred to \cref{subsec:Proof_theorem_comparison}).  
\begin{theorem}[Improvement of number of pulls of xID-SH in comparison to SH]
    \label{thm:ImprovementIDSH}
    Fix $n$ arms, a budget of $B$, a maximal size of $R$ and $r$ and $\eta$. Assume that we have already run SuccessiveHalving on $\tilde{n}$ arms and the same values for $B$, $r$ and $\eta$. Let $\eta_{-} = \eta-1$ and $s^{+} = s+1$. If we ran SuccessiveHalving (SH), efficient IterativeDeepening-SuccessiveHalving (eID-SH) and preserving resp.\ discarding IterativeDeepening-SuccessiveHalving (p/dID-SH) over $s$ rounds with above variables, we have
    \begin{align*}
        &\text{a) } \#\{\text{total pulls of eID-SH}\} \\
        &\leq \left( 1- \frac{(s^{+})(\tilde{n}R + \eta^s)(\eta_{-}) - (\eta^{s^{+}}-1)(2R+n)}{(s^{+})(nR+\eta^s)(\eta_{-}) - (\eta^{s^{+}}-1)(R+n)}\right)\\
        &~~~~~~\times \#\{\text{total pulls of SH}\}\\
        &\text{and b) } \#\{\text{total pulls of p/dID-SH}\} \\
        &\leq \left(1- \frac{(\eta_{-})((s^{+})\eta^s + R\tilde{n})-(\eta^{s^{+}}-1)(R+n)}{(\eta_{-})(s^{+})(nR+\eta^s)-(\eta^{s^{+}}-1)(R+n)}\right)\\
        &~~~~~~\times \#\{\text{total pulls of SH}\}.
    \end{align*}    
\end{theorem}
The fraction of improvement in the total number of pulls for eID-SH in comparison to SH is shown in Figure \ref{fig:ImprovementeID-SHeta2}, while for the other variants we provide similar plots in \cref{subsec:Proof_theorem_comparison}.
\begin{figure}[h!]
    \centering
    \includegraphics[width=0.9\columnwidth]{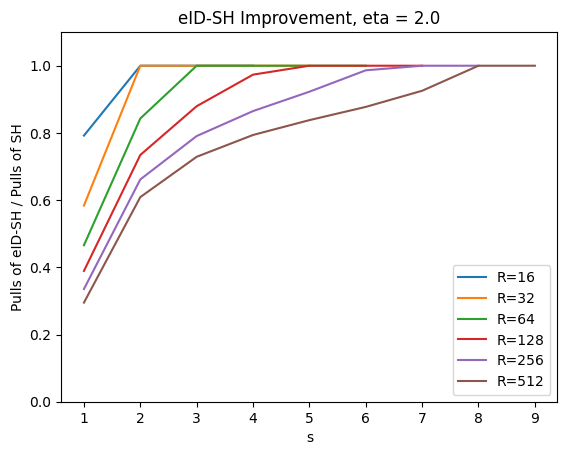}
    \caption{Fraction of number of pulls of eID-SH and SH for different values of rounds of SH $s$ and maximal budget per round $R$.}
    \label{fig:ImprovementeID-SHeta2}
\end{figure}
\subsection{IterativeDeepening-Hyperband}
An optimal hyperparameter configuration $\lambda^{\ast}$ as defined above may not always exist. Even if it exists, it could be infeasible to search for it as our hyperparameter configuration space is usually very large or even infinite. Therefore we will relax our goal and seek to find a configuration that is at least ``nearly optimal''. Similar to the HPO problem literature, we define the notion of such a near-optimal configuration as follows: For $\epsilon >0$, we call $\hat{\lambda}$ an \textit{$\epsilon$-optimal configuration} iff $\nu_{\hat{\lambda}} - \nu_{\lambda^{\ast}} \leq \epsilon.$ To ensure that the search for such a configuration is not like searching for the needle in the haystack, we need an assumption which guarantees that the probability of the existence of an $\epsilon$-optimal configuration in our sample set is high enough. 
\begin{assumption} \label{ass:proportion_best_configs}
    The proportion of $\epsilon$-optimal configurations in $\Lambda$ is $\alpha \in (0,1)$.
\end{assumption}
Note that we now have at least one $\epsilon$-optimal configuration in a sample set with probability at least $1-\delta$, if the size of the sample set is $\lceil \log_{1-\alpha}(\delta) \rceil$ for a fixed failure probability $\delta \in (0,1)$. With this, we can state the following theorem, the proof of which is given in \cref{subsec:theorem_IDHB}.
\begin{theorem}
    \label{thm:Budget_ID-HB}
    Let $\eta, R, \alpha$ and $\delta$ be fixed such that 
    \begin{align*}
        & R \geq \max\Big\{ \left\lceil \log_{1-\alpha}(\delta) \right\rceil (\eta -1) + 1 ,\\
        &  ~~~~~~\eta \Big( \log_{\eta}(\log_{\eta}(R)) + 4 + \frac{\lfloor \log_{\eta}(R) \rfloor}{2} \\
        &~~~~~~~ - \log_{\eta} \left( \left( \lfloor \log_{\eta}(R) \rfloor + 1 \right) ! \right)/ (\lfloor \log_{\eta}(R) \rfloor + 1) \Big) \bar{\gamma}^{-1} \Big\}\\
        &\text{for} ~~~\bar{\gamma}^{-1} := \max_{s = 0,\dots, \lfloor \log_{\eta}(R) \rfloor} \max_{i=2, \dots, n_s} i \Big( 1 \\
        &~~~~~~~~~~~~~ + \min \Big\{ R, \gamma^{-1} \Big( \max\Big\{ \frac{\epsilon}{4}, \frac{\nu_i- \nu_1}{2} \Big\} \Big) \Big\} \Big).
    \end{align*}
    Then ID-HB finds an $\epsilon$-optimal configuration with probability at least $1-\delta$.
\end{theorem}

\section{Empirical Evaluation}

In this section, we evaluate the proposed extension of \hb empirically and compare the three strategies devised in Section~\ref{sec:id-hb} to the original way of applying \hb when increasing the max size $R$ as done in the infinite horizon setting. More specifically, we are interested in the following two research questions:
\begin{description}
		[noitemsep,topsep=0pt,leftmargin=4mm]
    \item[RQ1] Is ID-HB able to retain the quality of returned hyperparameter configurations for its variants \discardingidhb, \conservativehb, and \efficienthb, respectively?
    \item[RQ2] To what extent can \discardingidhb and \conservativehb reduce the computational effort?
\end{description}

To answer \textbf{RQ1} and \textbf{RQ2}, we conduct an extensive set of experiments, the setup of which is outlined in Section~\ref{sec:experiment-setup}. The results of these experiments are subsequently presented and discussed in Section~\ref{sec:empirical-results}.

\subsection{Experiment Setup}\label{sec:experiment-setup}
In our experimental evaluation, we compare the proposed ID-HB approach in all its three flavors to the original \hb as a baseline to answer research questions \textbf{RQ1} and \textbf{RQ2}. To this end, we conduct an extensive set of experiments tackling various HPO tasks, including HPO for neural networks, SVMs, XGBoost, random forests, and neural architecture search.

As a benchmark library, we use YAHPO Gym \cite{DBLP:conf/automl/PfistererSMBB22} which provides fast-to-evaluate surrogate benchmarks for hyperparameter optimization with particular support for multi-fidelity optimization, which makes it a perfect fit for our study. From YAHPO Gym, we select the benchmarks listed in Table~\ref{tab:benchmark-scenarios}. Except for \texttt{nb301} only comprising CIFAR10 as a dataset, all other benchmarks include several datasets allowing for a broad comparison of ID-HB to the original \hb, subsequently denoted as IH-HB. We evaluate all benchmarks for $\eta=2$ and $\eta=3$, but due to space limitations only present results for $\eta=2$ here. Results for $\eta=3$ as well as detailed results for single datasets can be found in Appendix~\ref{sec:detailed-experimental-results}.

Furthermore, we set the initial max size $R_{t-1} = 16$ and increase it after the first run by a factor of $\eta$ to $R_t = R_{t-1} \eta$. For benchmarks considering a fraction of the training dataset as fidelity parameter, we translate a budget $r$ by $\nicefrac{r}{R_t}$ into a fraction between 0 and 1. Furthermore, we repeat each combination of algorithm, parameterization, and benchmark instance for 30 seeds resulting in a total amount of $30 \times 4 \times 2 \times 379 = 90,960$ hyperparameter optimization runs. We computed all experiments on a single workstation equipped with 2xIntel Xeon Gold 5122 and 256GB RAM.

The code and data are publicly available via GitHub\footnote{\url{https://github.com/mwever/iterative-deepening-hyperband}}.

\begin{table}[]
    \centering
    \caption{List of considered benchmarks from YAHPO-Gym, the type of learner, the number of considered datasets, the objective function, and the type of budget that can be used as a fidelity parameter.}
    \label{tab:benchmark-scenarios}
    \resizebox{\columnwidth}{!}{
    \begin{tabular}{c|c|c|c|c}
    \toprule
    Benchmark & Model & \# Inst. & Objective & Fidelity\\
    \midrule
        lcbench & neural network & 34 & val\_accuracy & epochs \\
        nb301 & neural network & 1 & val\_accuracy & epochs\\
        rbv2\_svm & SVM & 106 & acc & fraction \\
        rbv2\_ranger & random forest & 119 & acc & fraction \\
        rbv2\_xgboost & XGBoost & 119 & acc & fraction \\
        \bottomrule
    \end{tabular}}
\end{table}

\begin{figure*}[ht]
    \centering
    \begin{minipage}{.98\textwidth}
        \begin{minipage}{.28\textwidth}
            \begin{tikzpicture}
\begin{axis}[title={lcbench, $\eta=2$},xmin=0,xmax=1,ymin=0,ymax=1,legend pos=north west,width=5cm,height=5cm,tick style={grid=major},xlabel=IH-HB,ylabel=ID-HB,scatter/classes={eID-HB={mark=triangle,blue},pID-HB={mark=o,red},dID-HB={mark=square,magenta}}]
\addplot[scatter,only marks,scatter src=explicit symbolic] table[meta=label] {
x     y      label
0.8648454030355001  0.8648454030355001  eID-HB
0.9053111572265666  0.9053111572265666  eID-HB
0.8719565658569667  0.8716333491007666  eID-HB
0.9042846806844  0.9032897440592333  eID-HB
0.9444751942953334  0.9444751942953334  eID-HB
0.8179169031778667  0.8161305033365333  eID-HB
0.6901425272622334  0.6901425272622334  eID-HB
0.8000427195231666  0.7996776631673667  eID-HB
0.9610646591186334  0.9610646591186334  eID-HB
0.6361197802226  0.6371662635803667  eID-HB
0.5824315299986667  0.5824315299986667  eID-HB
0.8516643981934333  0.8522643814087334  eID-HB
0.9754554697672667  0.9753021494547667  eID-HB
0.9264976704915  0.9264976704915  eID-HB
0.9959782002766999  0.9959782002766999  eID-HB
0.28060517883296665  0.2768641789753667  eID-HB
0.7520440368651333  0.7533719635008667  eID-HB
0.8142776158651334  0.8150573374430666  eID-HB
0.7156172714233999  0.7156172714233999  eID-HB
0.8470576833089333  0.8453328119913667  eID-HB
0.9657115224202666  0.9654947662353667  eID-HB
0.8806141026814001  0.8771555074055333  eID-HB
0.9673818537394  0.9673818537394  eID-HB
0.7044126815795667  0.7042654266357  eID-HB
0.7970731048584  0.7970731048584  eID-HB
0.7148947550456334  0.7134573644002666  eID-HB
0.7798997777303334  0.7794399566650667  eID-HB
0.7377585957845  0.7377585957845  eID-HB
0.9289345575967  0.9289743550617  eID-HB
0.9222620620727666  0.9223959248861  eID-HB
0.9707560094198334  0.9702816848755667  eID-HB
0.6747893803915  0.6747063166301001  eID-HB
0.9812536926269666  0.9812536926269666  eID-HB
0.9934432220459667  0.9934364166260666  eID-HB
0.8648454030355001  0.8648454030355001  pID-HB
0.9053111572265666  0.9053111572265666  pID-HB
0.8719565658569667  0.8719565658569667  pID-HB
0.9444751942953334  0.9444751942953334  pID-HB
0.9042846806844  0.9042846806844  pID-HB
0.8179169031778667  0.8179169031778667  pID-HB
0.6901425272622334  0.6901425272622334  pID-HB
0.8000427195231666  0.8000427195231666  pID-HB
0.9610646591186334  0.9610646591186334  pID-HB
0.6361197802226  0.6371662635803667  pID-HB
0.5824315299986667  0.5824315299986667  pID-HB
0.8516643981934333  0.8534351704916  pID-HB
0.9754554697672667  0.9756691080729334  pID-HB
0.9264976704915  0.9264976704915  pID-HB
0.28060517883296665  0.2790411332448  pID-HB
0.9959782002766999  0.9959782002766999  pID-HB
0.7520440368651333  0.7533915888467333  pID-HB
0.8142776158651334  0.8150573374430666  pID-HB
0.7156172714233999  0.7156172714233999  pID-HB
0.8470576833089333  0.8472549921671667  pID-HB
0.9657115224202666  0.9657038879394667  pID-HB
0.8806141026814001  0.8788361562092667  pID-HB
0.9673818537394  0.9673818537394  pID-HB
0.7044126815795667  0.7042654266357  pID-HB
0.7970731048584  0.7970731048584  pID-HB
0.7148947550456334  0.7148947550456334  pID-HB
0.7377585957845  0.7377585957845  pID-HB
0.7798997777303334  0.7798997777303334  pID-HB
0.9289345575967  0.9289744745889  pID-HB
0.9222620620727666  0.9222620620727666  pID-HB
0.9707560094198334  0.9708169326783  pID-HB
0.6747893803915  0.6747893803915  pID-HB
0.9812536926269666  0.9812536926269666  pID-HB
0.9934432220459667  0.9934432220459667  pID-HB
0.8648454030355001  0.8648454030355001  dID-HB
0.9053111572265666  0.9053111572265666  dID-HB
0.8719565658569667  0.8719565658569667  dID-HB
0.9444751942953334  0.9444751942953334  dID-HB
0.9042846806844  0.9042846806844  dID-HB
0.8179169031778667  0.8179169031778667  dID-HB
0.6901425272622334  0.6901425272622334  dID-HB
0.8000427195231666  0.8000427195231666  dID-HB
0.9610646591186334  0.9610646591186334  dID-HB
0.6361197802226  0.6361197802226  dID-HB
0.5824315299986667  0.5824315299986667  dID-HB
0.8516643981934333  0.8516643981934333  dID-HB
0.9754554697672667  0.9754554697672667  dID-HB
0.9264976704915  0.9264976704915  dID-HB
0.28060517883296665  0.28060517883296665  dID-HB
0.9959782002766999  0.9959782002766999  dID-HB
0.7520440368651333  0.7520440368651333  dID-HB
0.8142776158651334  0.8142776158651334  dID-HB
0.7156172714233999  0.7156172714233999  dID-HB
0.8470576833089333  0.8470576833089333  dID-HB
0.9657115224202666  0.9657115224202666  dID-HB
0.8806141026814001  0.8806141026814001  dID-HB
0.9673818537394  0.9673818537394  dID-HB
0.7044126815795667  0.7044126815795667  dID-HB
0.7970731048584  0.7970731048584  dID-HB
0.7148947550456334  0.7148947550456334  dID-HB
0.7798997777303334  0.7798997777303334  dID-HB
0.7377585957845  0.7377585957845  dID-HB
0.9289345575967  0.9289345575967  dID-HB
0.9222620620727666  0.9222620620727666  dID-HB
0.9707560094198334  0.9707560094198334  dID-HB
0.6747893803915  0.675143237936138  dID-HB
0.9812536926269666  0.9812536926269666  dID-HB
0.9934432220459667  0.9934432220459667  dID-HB
};\addlegendentry{eID-HB}
\addlegendentry{pID-HB}
\addlegendentry{dID-HB}
\addplot[color=black] coordinates {
	(0,0)
	(1,1)
};
\end{axis}
\end{tikzpicture}
        \end{minipage}
        \begin{minipage}{.21\textwidth}
            \input{scatter-plots-rbv2-ranger-2.tex}
        \end{minipage}
        \begin{minipage}{.21\textwidth}
            \input{scatter-plots-rbv2-svm-2.tex}
        \end{minipage}
        \begin{minipage}{.21\textwidth}
            \input{scatter-plots-rbv2-xgboost-2.tex}
        \end{minipage}
        \vspace{-.8cm}
    \end{minipage}
    \begin{minipage}{.98\textwidth}
        \begin{minipage}{.28\textwidth}
            \hfill
            \includegraphics[width=.75\textwidth]{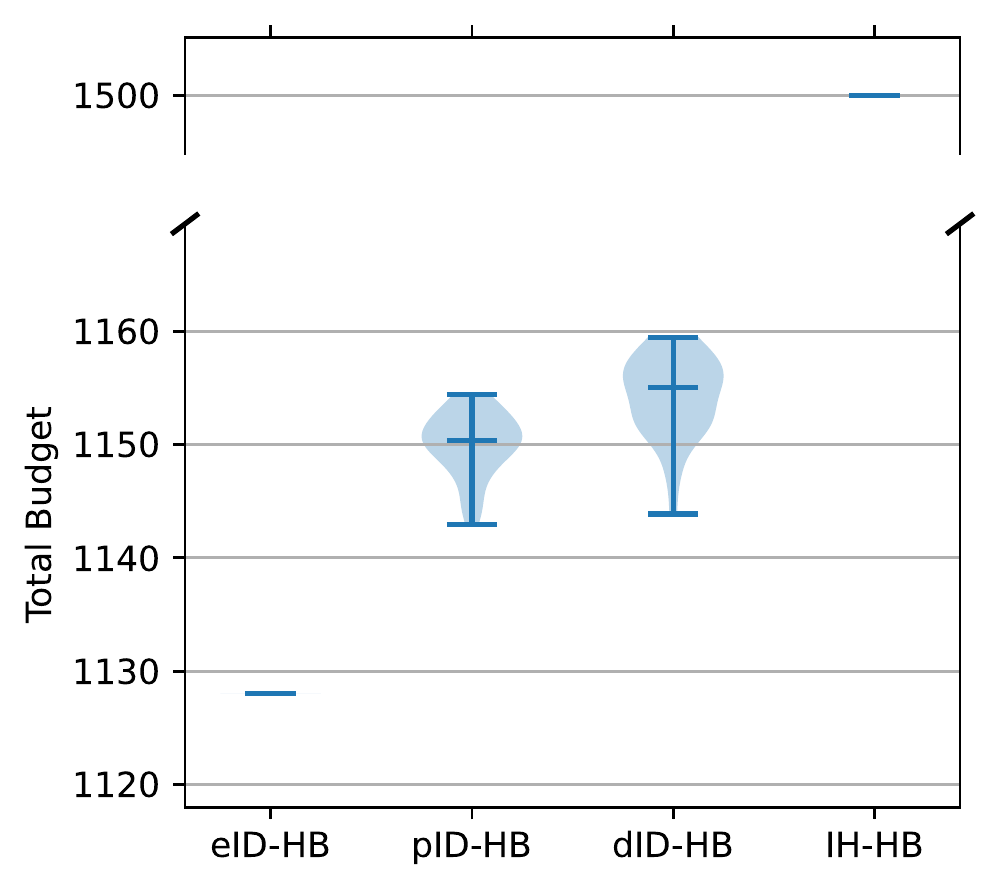}
        \end{minipage}
        \begin{minipage}{.21\textwidth}
            \includegraphics[width=\textwidth]{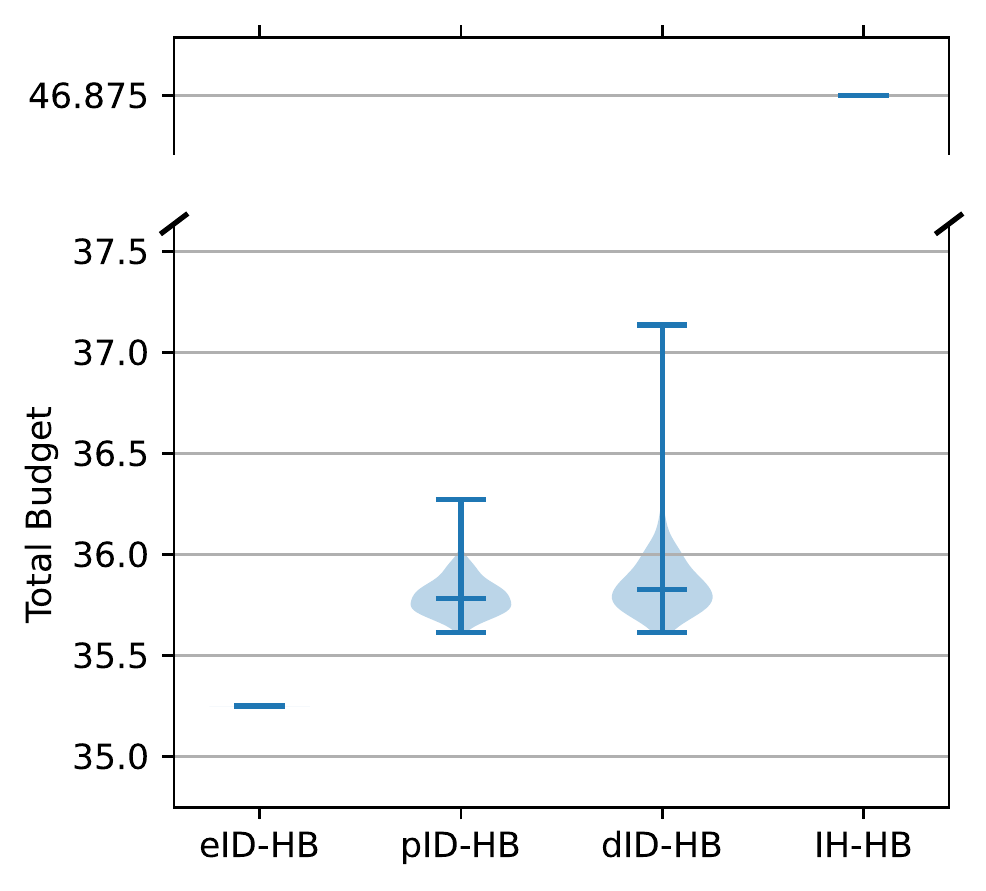}
        \end{minipage}
        \begin{minipage}{.21\textwidth}
            \includegraphics[width=\textwidth]{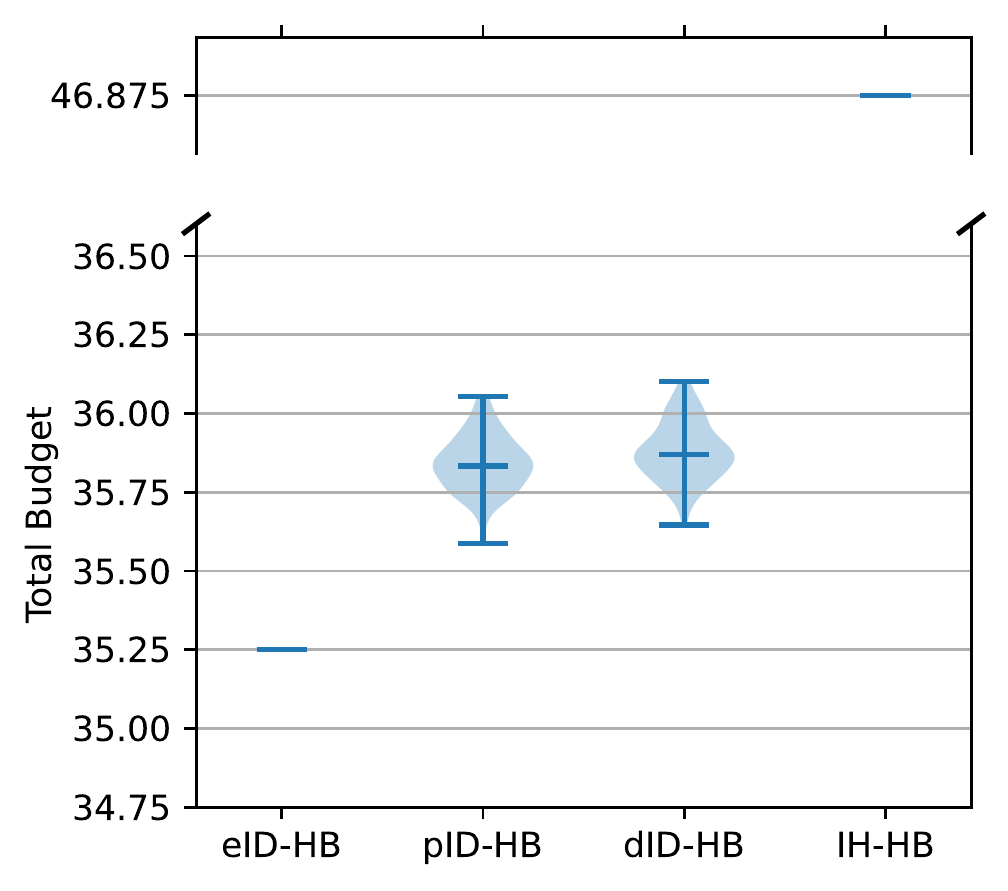}
        \end{minipage}
        \begin{minipage}{.21\textwidth}
            \includegraphics[width=\textwidth]{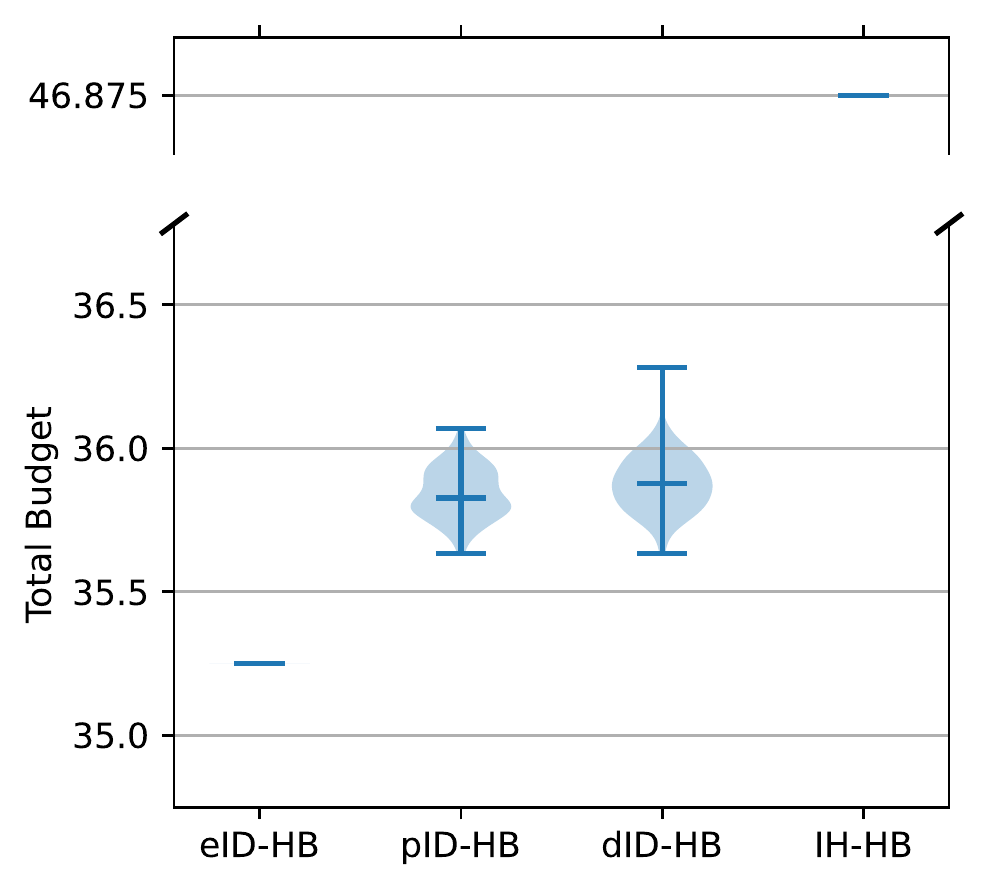}
        \end{minipage}
    \end{minipage}
    \caption{Comparison of ID-HB to the original version of \hb. Top: Scatter plots plotting the final incumbents' accuracy obtained by an ID-HB strategy versus IH-HB. Bottom: Violin plots showing the average total budget consumed for a single run.}
    \label{fig:result-plots}
\end{figure*}

\subsection{Empirical Results}\label{sec:empirical-results}

In Figure~\ref{fig:result-plots}, we present the experimental results for the benchmarks with more than one dataset \texttt{lcbench}, \texttt{rbv2\_svm}, \texttt{rbv2\_range}, and \texttt{rbv2\_xgboost}. In the top row we present scatter plots comparing the final incumbent's performance obtained by IH-HB on the x-axis against the performance obtained by an \idhb strategy on the y-axis. The diagonal line represents the situation that both performances are on a par. From these plots, it is quite obvious that any ID-HB strategy is performing similarly well as the original \hb variant. Similar observations can be made for the \texttt{nb301} benchmark, where all approaches achieve an accuracy of $0.9061$ (see \cref{sec:detailed-experimental-results}). 

Concerning \textbf{RQ1}, for the benchmarks considered here, we can thus confirm empirically that the proposed \idhb extension is indeed able to retain the quality of the returned hyperparameter configurations. Especially notable is the performance of \efficienthb, as it does not revise any decisions and thus would intuitively be likely to show deterioration in performance; yet, in the hyperparameter optimization considered here, this is never the case. Without any exception, all \hb variants perform equally well with respect to the quality of returned solutions. Furthermore, the results are also representative for $\eta=3$.

Considering the average budget consumed for a single run, however, significant improvements can be achieved with \idhb. As can be seen from the bottom row in Figure~\ref{fig:result-plots}, which display the distribution of the total budget consumed during a single run, \efficienthb represents the most efficient strategy. Confirming intuitive expectations, \conservativehb is more efficient than \discardingidhb, although the differences are often rather marginal. Perhaps more surprisingly, both \conservativehb and \discardingidhb are significantly more efficient than the IH-HB in all benchmarks. Even the worst runs still yield a 20\% reduction of the total budget consumed, answering \textbf{RQ2}. Although the theoretical worst case analysis for \conservativehb and \discardingidhb gives only slide, almost negligible improvements, in practice, these strategies seem to be quite efficient and revise only few evaluations.

\section{Conclusion and Future Work}
In this paper, we have proposed an extension to the well-known HPO method \hb called Iterative Deepening Hyperband (ID-HB) aiming to improve its efficiency when the max size hyperparameter of \hb needs to be increased post-hoc. We derived three strategies with varying truthfulness with respect to run \hb from scratch on the same sample of hyperparameter configurations. For all of these three strategies, we gave theoretical guarantees on the quality of the final choice as well as on the saved budget when a previously \hb run is continued. In an empirical study, we also find all three strategies to yield similar results as the much more expensive baseline variant of \hb. In fact, in the most efficient strategy, our approach only requires the budget of the one run with the increased max size. In future work, we plan to combine our more efficient \hb extensions with more sophisticated sampling of hyperparameter configurations as for example done in \cite{DBLP:conf/ijcai/AwadMH21} or \cite{DBLP:conf/icml/FalknerKH18} and HyperJump to improve ID-HB's efficacy and efficiency even more.

\section*{Software and Data}

The software and experimental data ispublicly available via GitHub: \url{https://github.com/mwever/iterative-deepening-hyperband}.

\section*{Acknowledgements}
This research was supported by the research training group Dataninja (Trustworthy AI for Seamless Problem Solving: Next Generation Intelligence Joins Robust Data Analysis) funded by the German federal state of North Rhine-Westphalia.


\bibliography{literature}
\bibliographystyle{icml2023}

\newpage
\appendix
\onecolumn
\section{Pseudo-Codes}
For the theoretical analysis we consider the following algorithms where the Hyperband and the Successive Halving parts are strictly seperated. All parts of the algorithms that are different to the original Successive Halving algorithm by \cite{JamiesonSuccHalv} and to the \hb algorithm by \cite{LiHyperband} are indicated by a blue text color.

\begin{algorithm*}[ht!]
    \caption{IterativeDeepening-Hyperband (ID-HB)}
    \label{alg:ID-HB}
    \begin{algorithmic}
        \STATE \textbf{Input:} max size $R$, $\eta\geq 2$, \color{blue}{old max size $\tilde{R} \in \{0, R/\eta\}$, old sequence of configuration samples $((C_{s,k})_{k\in \{0,\dots,s\}})_{s\in \{0,\dots,\log_{\eta}(\tilde{R})\}}$ and losses $((L_{s,k})_{k\in \{0,\dots,s\}})_{s\in \{0,\dots,\log_{\eta}(\tilde{R})\}}$}\color{black} 
        \STATE \textbf{Initialize:} $s_{max}=\lfloor \log_{\eta}(R) \rfloor$,
        $B=(s_{max}+1)R$
        \color{blue}
        \IF{$\tilde{R}>0$} 
            \STATE $\tilde{s}_{max}=\lfloor \log_{\eta}(\tilde{R}) \rfloor = s_{max}-1$, 
            $\tilde{B} = (\tilde{s}_{max}+1)\tilde{R}$
        \ENDIF
        \color{black}
        \FOR{$s \in \{s_{max}, s_{max}-1, \dots, 0\}$}
            \STATE $n_{s} = \lceil \frac{B}{R}\frac{\eta^{s}}{(s+1)} \rceil$,
            $r_{s} = R/\eta^{s}$
            \color{blue}
            \IF{$\tilde{R}>0$ and $s>0$}
                \STATE $\tilde{s}=s-1$,
                $\tilde{n}_{s} = \lceil \frac{\tilde{B}}{\tilde{R}}\frac{\eta^{\tilde{s}}}{(\tilde{s}+1)} \rceil$,
                $\tilde{r}_{s} = \tilde{R}/\eta^{\tilde{s}} = r_s$
            \ELSE
                \STATE $\tilde{n}_{s} = 0$
            \ENDIF
            \color{black}
            \STATE $S \leftarrow$ sample $n_s - \color{blue} \tilde{n}_s$ \color{black} configurations
            \STATE \text{xID-SuccessiveHalving($S, B, r_s, R, \eta, \color{blue} (C_{\tilde{s},k})_{k\in \{0,\dots,\tilde{s}\}}, (L_{\tilde{s},k})_{k\in \{0,\dots,\tilde{s}\}}$\color{black})}
        \ENDFOR
        \STATE \textbf{Output:} \text{Configuration with smallest intermediate loss}
    \end{algorithmic}
\end{algorithm*}
\begin{algorithm*}[ht!]
    \caption{Efficient IterativeDeepening-SuccessiveHalving (eID-SH)}
    \label{alg:eID-SH}
    \begin{algorithmic}
        \STATE \textbf{Input:} $S$ set of arms, budget $B$, $r$, max size $R$, 
        $\eta$, \color{blue} $(C_{k})_{k}$ old sequence of configurations, $(L_{k})_{k}$ old sequence of losses \color{black} 
        \STATE \textbf{Initialize:} $S_{0} \leftarrow S$, \color{blue} $\tilde{n} = |C_0|$\color{black}, $n=|S_0| + |C_0|$
        %
        \FOR{$k \in \{0, 1, \dots, s\}$}
            \STATE $n_{k} = \lfloor n/\eta^{k} \rfloor - \color{blue} \lfloor \tilde{n}/\eta^{k} \rfloor$ \color{black},
            $r_{k} = r\eta^{k}$
            \STATE pull each arm in $S_{k}$ for $r_{k}$ times
            \IF{$k\leq s-1$}
                \STATE $S_{k+1} \leftarrow $ keep the best $\lfloor n / \eta^{k+1} \rfloor - \color{blue} \lfloor \tilde{n} / \eta^{k+1} \rfloor$ \color{black} arms from \color{blue}$S_k \cup C_{k} \backslash C_{k+1}$ \color{black}
            \ELSE
                \STATE $S_{k+1} \leftarrow $ keep the best $\lfloor n / \eta^{k+1} \rfloor$ \color{black} arms from \color{blue} $S_k \cup C_k$ \color{black}
            \ENDIF
        \ENDFOR
        \STATE \textbf{Output:} \text{Remaining configuration}
    \end{algorithmic}
\end{algorithm*}
\begin{algorithm*}[ht!]
    \caption{Discarding IterativeDeepening-SuccessiveHalving (dID-SH)}
    \label{alg:dID-SH}
    \begin{algorithmic}
        \STATE \textbf{Input:} $S$ set of arms, budget $B$, $r$, max size $R$, 
        $\eta$, \color{blue} $(C_{k})_{k}$ old sequence of configurations, $(L_{k})_{k}$ old sequence of losses \color{black}
        \STATE \textbf{Initialize:} \color{blue} $S_{0} \leftarrow S \cup C_0$\color{black}, $n=|S_0|$
        %
        \FOR{$k \in \{0, 1, \dots, s\}$}
            \STATE $n_{k} = \lfloor n/\eta^{k} \rfloor$,
            $r_{k} = r\eta^{k}$
            \STATE pull each arm in \color{blue}$S_{k} \backslash C_{k}$ \color{black} for $r_{k}$ times
            \STATE $S_{k+1} \leftarrow $ keep the best $\lfloor n / \eta^{k+1} \rfloor$ arms from $S_k$
        \ENDFOR
        \STATE \textbf{Output:} \text{Remaining configuration}
    \end{algorithmic}
\end{algorithm*}
\begin{algorithm*}[ht!]
    \caption{Preserving IterativeDeepening-SuccessiveHalving (pID-SH)}
    \label{alg:cID-SH}
    \begin{algorithmic}
        \STATE \textbf{Input:} $S$ set of arms, budget $B$, $r$, max size $R$, 
        $\eta$, \color{blue} $(C_{k})_{k}$ old sequence of configurations, $(L_{k})_{k}$ old sequence of losses \color{black}
        \STATE \textbf{Initialize:} \color{blue} $S_{0} \leftarrow S \cup C_0$\color{black}, $n=|S_0|$
        %
        \FOR{$k \in \{0, 1, \dots, s\}$}
            \STATE $n_{k} = \lfloor n/\eta^{k} \rfloor$,
            $r_{k} = r\eta^{k}$
            \STATE pull each arm in \color{blue}$S_{k} \backslash C_{k}$ \color{black} for $r_{k}$ times
            \STATE $S_{k+1} \leftarrow $ keep the best $\lfloor n / \eta^{k+1} \rfloor$ arms from \color{blue} $S_k \cup C_k$ \color{black}
        \ENDFOR
        \STATE \textbf{Output:} \text{Remaining configuration}
    \end{algorithmic}
\end{algorithm*}

\section{Proofs}
%
\subsection{Proof of Theorem \ref{thm:Budget_ID-SH}} \label{subsec:Proof_necessary_budget}
\begin{proof}[Proof of Theorem \ref{thm:Budget_ID-SH}]
    First, we will focus on the efficient IterativeDeepening-SuccessiveHalving Algorithm given in Algorithm \ref{alg:eID-SH}.\\
    \underline{Step 1:} Algorithm \ref{alg:eID-SH} never exceeds the budget $B$, which can be seen as follow. The budget used is bounded by
    \begin{align*}
        \sum_{k=0}^{s} n_{k} r_{k} &= \sum_{k=0}^{s} \left( \lfloor n / \eta^{k} \rfloor - \lfloor \tilde{n} / \eta^{k} \rfloor \right) \frac{R\eta^{k}}{\eta^{s}} \\
        &\leq \sum_{k=0}^{s} \left( \lfloor n / \eta^{k} \rfloor  \right) \frac{R\eta^{k}}{\eta^{s}} \\
        &\leq \sum_{k=0}^{s}   \frac{(s_{\max}+1) \eta^s}{(s+1)}    \frac{R }{\eta^{s}} \\
        &\leq     (s_{\max}+1)  R  = B.\\
        %
        %
    \end{align*}
    %
    %
    %
    %
    
    \underline{Step 2:} Let $n_k = |S_{k}| + |C_{k}|$ and $\tilde{n}_k = |C_{k}|$ such that $n_0 = n$ and $\tilde{n}_o = \tilde{n}$.
    Without loss of generality, we assume that the limit values of the losses are ordered, such that $\nu_1\leq \nu_2\leq\dots\leq\nu_n$. 
    Note, that due to the above condition also the limit values of arms in $S_k$ and resp. in $C_k$ are ordered, e.g. for $\nu_i, \nu_j \in S_k$ with $i<j$ we have $\nu_i \leq \nu_j$.
    Let in the following be $i'_{k}= \min\left\{\lfloor n_k/\eta \rfloor +1, \lfloor \tilde{n}_k/\eta \rfloor + \lfloor n_k/ \eta \rfloor + 1 \right\}$.
    Assume that $B\geq z_{eID-SH}$, then we have for each round $k$
    \begin{align*}
        r_k &\geq \frac{B}{(n_k-\tilde{n}_k) \lceil \log_{\eta}(n) \rceil}-1  \\
        &\geq \frac{\eta}{n_k-\tilde{n}_k}\max_{i=2,\dots,n} i\bigg(1+\min\big\{R,\gamma^{-1}\big(\max\big\{\frac{\epsilon}{4},\frac{\nu_i-\nu_1}{2}\big\}\big)\big\}\bigg) -1 \\
        &\geq \frac{\eta}{n_k-\tilde{n}_k} \ i'_k \bigg(1+\min\big\{R,\gamma^{-1}\big(\max\big\{\frac{\epsilon}{4},\frac{\nu_{i'_k}-\nu_1}{2}\big\}\big)\big\}\bigg) -1 \\
        &\overset{(*)}{\geq} \frac{\eta}{n_k-\tilde{n}_k} \ (n_k-\tilde{n}_k)/\eta  \bigg(1+\min\big\{R,\gamma^{-1}\big(\max\big\{\frac{\epsilon}{4},\frac{\nu_{i'_k}-\nu_1}{2}\big\}\big)\big\}\bigg) -1 \\
        &= \bigg(1+\min\big\{R,\gamma^{-1}\big(\max\big\{\frac{\epsilon}{4},\frac{\nu_{i'_k}-\nu_1}{2}\big\}\big)\big\}\bigg) - 1\\
        &= \min\big\{R,\gamma^{-1}\big(\max\big\{\frac{\epsilon}{4},\frac{\nu_{i'_k}-\nu_1}{2}\big\}\big)\big\},
    \end{align*}
    where the fourth line $(*)$ follows from:
    \begin{itemize}
        \item \textbf{Case 1:} $i'_k = \lfloor \tilde{n}_k/\eta \rfloor +1$.\\
        We have $i'_k \geq n_k/\eta \geq (n_k - \tilde{n}_k)/\eta$.
        \item \textbf{Case 2:} $i'_k = \lfloor \tilde{n}_k/\eta \rfloor + \lfloor n_{k-1}/ \eta \rfloor +1$.\\
        If $\tilde{n}_k = 0$, we have 
        \begin{align*}
            i'_k = \left\lfloor \frac{n_{k-1}}{\eta} \right\rfloor +1 \geq \frac{n_{k-1}}{\eta} \geq \frac{n_k}{\eta} \geq \frac{n_k - \tilde{n}_k}{\eta}.
        \end{align*}
        If $\tilde{n}_k \geq 1$, we have
        \begin{align*}
            i'_k &\geq \frac{\tilde{n}_k}{\eta} - 1 + \frac{n_{k-1}}{\eta} - 1 + 1 \\
            &= \frac{\tilde{n}_k - n_{k-1} - \eta}{\eta} \\
            &\geq \frac{n_k + (\eta-1) n_{k} + \tilde{n}_k - \eta}{\eta} \\
            &\geq \frac{n_k}{\eta} \geq \frac{n_k - \tilde{n}_k}{\eta},
        \end{align*}
        where line 3 follows from $n_k = \lfloor n_{k-1}/\eta \rfloor$ and line 4 from $n_k \geq \tilde{n}_k \geq 1$ and $\eta \geq 2$, so we have $\eta-1 \geq 1$, so we can estimate $n_k \geq 1$.
    \end{itemize}
    Next, we show that $\ell_{i,t}-\ell_{1,t} > 0$ for all $t \geq \tau_i := \gamma^{-1}\big(\frac{\nu_i-\nu_1}{2}\big)$.
    Given the definition of $\gamma$, we have for all $i \in [n]$ that $|\ell_{i,t} - \nu_i| \leq \gamma(t) \leq \frac{\nu_i-\nu_1}{2}$ where the last inequality holds for $t \geq \tau_i$. 
    Thus, for $t \geq \tau_i$ we have
    \begin{align*}
        \ell_{i,t} - \ell_{1,t} &= \ell_{i,t} - \nu_i +\nu_i -\nu_1 + \nu_1 - \ell_{1,t} \\
        & = \ell_{i,t} - \nu_i - (\ell_{1,t} - \nu_1) + \nu_i - \nu_1 \\
        & \geq -2 \gamma(t) +\nu_i - \nu_1 \\
        & \geq -2\frac{\nu_i-\nu_1}{2} +\nu_i - \nu_1\\
        &= 0.
    \end{align*}
    Under this scenario, we will eliminate arm $i$ before arm $1$ since on each round the arms are sorted by their empirical losses and the top half are discarded.
    Note that by the assumption the $\nu_i$ limits are non-decreasing in $i$ so that the $\tau_i$ values are non-increasing in $i$.

    Fix a round $k$ and assume $1 \in S_k \cup C_{k}$ (note, $1\in S_0 \cup C_{0}$).
    The above calculation shows that  
    \begin{align}
        t \geq \tau_i \implies \ell_{i,t} \geq \ell_{1,t}. \label{eqn:suff_emp_losses}
    \end{align} 
    We regard two different scenarios in the following.
    \begin{itemize}
        \item \textbf{Case 1:} $k \leq s-1$.\\
        In this case, we keep the best $\lfloor n_k/ \eta \rfloor - \lfloor \tilde{n}_k/ \eta \rfloor$ arms from the set $S_k \cup C_{k} \backslash C_{k+1}$ and have already promoted the best $\lfloor \tilde{n}_k/ \eta \rfloor$ from $C_k$.
        \begin{align*}
            \{1 \in S_k \cup C_{k}, \ \ &1 \notin S_{k+1}\cup C_{k+1} \} \\
            \iff& \left\{ \sum_{i \in S_k \cup C_{k} \backslash C_{k+1}} \1\{ \ell_{i,r_k} < \ell_{1,r_k} \} \geq \lfloor n_k/\eta \rfloor - \lfloor \tilde{n}_k/\eta \rfloor, \sum_{i \in C_k} \1\{ \ell_{i,r_k} < \ell_{1,r_k} \} \geq \lfloor \tilde{n}_k/\eta \rfloor \right\} \\
            \implies& \left\{ \sum_{i \in S_k \cup C_{k} \backslash C_{k+1}} \1\{ r_k < \tau_i \} \geq \lfloor n_k/\eta \rfloor - \lfloor \tilde{n}_k/\eta \rfloor, \sum_{i \in C_k} \1\{ r_k < \tau_i \} \geq \lfloor \tilde{n}_k/\eta \rfloor \right\} \\
            \implies& \Bigg\{ \sum_{i = 2}^{\lfloor n_k/\eta \rfloor - \lfloor \tilde{n}_k/\eta \rfloor + \lfloor \tilde{n}_{k}/\eta \rfloor + 1} \1\{ r_k < \tau_i  \wedge i \in S_{k} \cup C_{k} \backslash C_{k+1}\} \geq \lfloor n_k/\eta \rfloor - \lfloor \tilde{n}_k/\eta \rfloor, \\
            &~~ \sum_{i = 2}^{\lfloor \tilde{n}_k/\eta \rfloor + \lfloor n_{k-1}/\eta \rfloor +1} \1\{ r_k < \tau_i  \wedge i \in C_{k}\} \geq \lfloor \tilde{n}_k/\eta \rfloor \Bigg\} \\
            \implies& \left\{ r_k < \min \left\{\tau_{\lfloor n_k/\eta \rfloor +1},\tau_{\lfloor \tilde{n}_k/\eta \rfloor + \lfloor n_{k-1}/\eta \rfloor + 1}  \right\} \right\}\\
            \iff& \left\{ r_k < \tau_{\max \left\{\lfloor n_k/\eta \rfloor +1, \lfloor \tilde{n}_k/\eta \rfloor + \lfloor n_{k-1}/\eta \rfloor + 1 \right\} }  \right\},
        \end{align*}
        %
        where the first line follows by the definition of the algorithm and the second by Equation~\ref{eqn:suff_emp_losses}. In the third line we assume the worst case scenario, where the best $\lfloor n_k/ \eta \rfloor - \lfloor \tilde{n}_k/ \eta \rfloor$ arms in $S_k \cup C_{k} \backslash C_{k+1}$ are all worse than the best $\lfloor \tilde{n}_{k}/\eta \rfloor$ arms in $C_{k}$ (which are kept in the set $C_{k+1}$) and vice versa that the best $\lfloor \tilde{n}_{k}/\eta \rfloor$ arms in $C_k$ are worse than all arms in $S_k$. The fourth line follows by $\tau_i$ being non-increasing (for all $i < j$ we have $\tau_i \geq \tau_j$ and consequently, $\1\{ r_k < \tau_i \} \geq \1\{ r_k < \tau_j \}$ so the \textit{first} indicators of the sum not including $1$ would be on before any other $i$'s in $S_k \subset [n]$ sprinkled throughout $[n]$).
        \item \textbf{Case 2:} $k = s$.\\
        In this case we keep the best $\lfloor n_k/\eta \rfloor$ arms from $S_{k} \cup C_{k}$ and have $C_{k+1} = \emptyset$, thus we get analogously as above
        \begin{align*}
            \{1 \in S_k \cup C_{k}, \ \ 1 \notin S_{k+1} \} \iff& \left\{ \sum_{i \in S_k \cup C_{k}} \1\{ \ell_{i,r_k} < \ell_{1,r_k} \} \geq \lfloor n_k/\eta \rfloor \right\} \\
            \implies& \left\{ \sum_{i \in S_k \cup C_{k}} \1\{ r_k < \tau_i \} \geq \lfloor n_k/\eta \rfloor \right\} \\
            \implies& \left\{ \sum_{i = 2}^{\lfloor n_k/\eta \rfloor +1} \1\{ r_k < \tau_i \} \geq \lfloor n_k/\eta \rfloor \right\} \\
            \iff& \left\{ r_k < \tau_{\lfloor n_k/\eta \rfloor +1} \right\}.
        \end{align*}
    \end{itemize}
    Overall, we can conclude, that $1 \in S_k \cup C_{k}$ and $1 \notin S_{k+1}\cup C_{k+1}$ if \\
    $r_k < \tau_{\max\left\{\lfloor n_k/\eta \rfloor +1, \lfloor \tilde{n}_k/\eta \rfloor + \lfloor n_k/ \eta \rfloor + 1 \right\}}$.
    This implies
    \begin{align}
        \{ 1 \in S_k \cup C_{k}, \ \ r_k \geq \tau_{i'_k} \} \implies \{ 1 \in S_{k+1} \cup C_{k+1} \}. \label{eqn:recursive_keep}
    \end{align}

    Recalling that $r_k \geq \gamma^{-1}\big(\max\big\{\frac{\epsilon}{4},\frac{\nu_{i'_k}-\nu_1}{2}\big\}\big)$ and \\
    $\tau_{i'_k} =  \gamma^{-1}\big(\frac{\nu_{i'_k}-\nu_1}{2}\big)$, we examine the following three exhaustive cases:
    \begin{itemize}
        \item \textbf{Case 1:} $\frac{\nu_{i'_k}-\nu_1}{2} \geq \frac{\epsilon}{4}$ and $1 \in S_k \cup C_{k}$
           
        In this case, $r_k \geq \gamma^{-1}\big( \frac{\nu_{i'_k}-\nu_1}{2}\big) = \tau_{i'_k}$.
        By Equation~\ref{eqn:recursive_keep} we have that $1 \in S_{k+1} \cup C_{k+1}$ since $1 \in S_k \cup C_{k}$.

        \item \textbf{Case 2:} $\frac{\nu_{i'_k}-\nu_1}{2} < \frac{\epsilon}{4}$ and $1 \in S_k \cup C_{k}$
    
        In this case $r_k \geq \gamma^{-1}\big(\frac{\epsilon}{4}\big)$ but $\gamma^{-1}\big(\frac{\epsilon}{4}\big) < \tau_{i'_k}$. 
        Equation~\ref{eqn:recursive_keep} suggests that it may be possible for $1 \in S_k \cup C_{k}$ but $1 \notin S_{k+1} \cup C_{k+1}$.
        On the good event that $1 \in S_{k+1} \cup C_{k+1}$, the algorithm continues and on the next round either case 1 or case 2 could be true. 
        So assume $1 \notin S_{k+1} \cup C_{k+1}$.
        Here we show that $\{1 \in S_k \cup C_{k}, \ \ 1 \notin S_{k+1} \cup C_{k+1}\} \implies \max_{i \in S_{k+1} \cup C_{k+1}} \nu_i \leq \nu_1 +\epsilon/2$.
        Because $1 \in S_0 \cup C_{0}$, this guarantees that Algorithm \ref{alg:eID-SH} either exits with arm $\widehat{i} = 1$ or some arm $\widehat{i}$ satisfying $\nu_{\widehat{i}} \leq \nu_1 + \epsilon/2$.
    
        Let $p = \min \{i \in [n] : \frac{\nu_i -\nu_1}{2} \geq \frac{\epsilon}{4}\}$.  Note that $p>i'_k$ by the criterion of the case and 
        \[r_k \geq \gamma^{-1}\left(\frac{\epsilon}{4}\right) \geq \gamma^{-1}\left(\frac{\nu_i-\nu_1}{2}\right) = \tau_i, \quad \forall i \geq p.\]
        Thus, by Equation~\ref{eqn:suff_emp_losses} ($t \geq \tau_i \implies \ell_{i,t} \geq \ell_{1,t}$) we have that arms $i\geq p$ would always have $\ell_{i,r_k} \geq \ell_{1,r_k}$ and be eliminated before or at the same time as arm $1$, presuming $1 \in S_k \cup C_{k}$. 
        In conclusion, if arm $1$ is eliminated so that $1 \in S_k \cup C_{k}$ but $1 \notin S_{k+1} \cup C_{k+1}$ then $\max_{i \in S_{k+1} \cup C_{k+1}} \nu_i \leq \max_{i < p} \nu_i < \nu_1 +\epsilon/2$ by the definition of $p$.

        \item \textbf{Case 3:} $1 \notin S_k \cup C_{k}$
    
        Since $1 \in S_0 \cup C_{0}$, there exists some $r <k$ such that $1 \in S_r \cup C_{r}$ and $1 \notin S_{r+1} \cup C_{r+1}$. 
        For this $r$, only case 2 is possible since case 1 would proliferate $1 \in S_{r+1} \cup C_{r+1}$. 
        However, under case 2, if $1 \notin S_{r+1} \cup C_{r+1}$ then $\max_{i \in S_{r+1} \cup C_{r+1}} \nu_i \leq \nu_1 +\epsilon/2$.
    
    \end{itemize}
    
    Because $1 \in S_0 \cup C_{0}$, we either have that $1$ remains in $S_k \cup C_{k}$ (possibly alternating between cases 1 and 2) for all $k$ until the algorithm exits with the best arm $1$, or there exists some $k$ such that case 3 is true and the algorithm exits with an arm $\widehat{i}$ such that $\nu_{\widehat{i}} \leq \nu_1 + \epsilon/2$.\\

    %
%
%
%
%
    %
    Next, we proof the same guarantee for the discarding and preserving IterativeDeepening-SuccessiveHalving algorithms given in Algorithm \ref{alg:dID-SH} and Algorithm \ref{alg:cID-SH}.\\
    Therefore we proceed in two steps: First, we will reduce the dID-SH algorithm to the SH algorithm to take over its theoretical guarantees. Second, we will show where the proof of SH has to be modified to achieve the same theoretical guarantees for our pID-SH algorithm.\\
    
    \underline{Step 1:} We will distinguish two different cases in the following in order to reduce the discarding IterativeDeepening-SuccessiveHalving algorithm \ref{alg:dID-SH} to the original version of Successive Halving by \cite{JamiesonSuccHalv} (or \cite{karnin2013almost}).
    \begin{itemize}
        \item \textbf{Case 1:} $(C_k)_k = \emptyset$.\\
        If we have $(C_k)_k = \emptyset$, we have simply the Successive Halving algorithm by \cite{JamiesonSuccHalv} and can keep their theoretical guarantees. 
        \item \textbf{Case 2:} $(C_k)_k \neq \emptyset$.\\
        Thus the interesting case which we will consider in the following is the case $(C_k)_k \neq \emptyset$. Assume that Algorithm \ref{alg:dID-SH} is called as subroutine by Algorithm \ref{alg:ID-HB}. Since $(C_k)_k \neq \emptyset$, Algorithm \ref{alg:dID-SH} was already called before with number of arms $\tilde{n}$ and budget $\tilde{r}_s = \tilde{R}/\eta^{\tilde{s}} = \frac{R}{\eta} / \eta^{s-1} = R / \eta^{s} = r_k$ for $s \in \{0,\dots, \lfloor \log_{\eta}(R) \rfloor\}$. Thus, the arms in $(C_k)_k$ were already pulled for $r_k$ times and their loss values $(L_k)_k$ were observed. Combining these with the loss values we observe in each iteration $k$ in Algorithm \ref{alg:dID-SH} for $r_k$ pulls of the arms in $S_k \backslash C_k$, we can keep the best $\lfloor n/ \eta^{k+1} \rfloor$ arms from $S_k$ regarding the observed losses of the recent pulls of $S_k \backslash C_k$ and the before observed losses of $C_k$. Therefore, we get the same arms in $S_{k+1}$ as starting Algorithm \ref{alg:cID-SH} from scratch with $(C_k)_k = \emptyset$ and $S = S \cup C_0$ and can apply Case 1. 
    \end{itemize}
    To conclude both cases, we can keep the theoretical result that was proven by \cite{LiHyperband} for the original version of Successive Halving in a finite horizon setting ($R < \infty$).\\
    
    %
    \underline{Step 2:} To achieve the same guarantee for the preserving IterativeDeepening-SuccessiveHalving algorithm, we can proceed analogue as in the proof of Successive Halving by \cite{LiHyperband}. For a fixed round $k$ and $1\in S_k \cup C_k$, since $1\in S_0 \cup C_0$, we have
    \begin{align*}
        \{ 1 \in S_k \cup C_k, 1 \notin S_{k+1} \} &\Leftrightarrow \left\{ \sum_{i \in S_k \cup C_k} \1\{ \ell_{i,r_k} < \ell_{1,r_k} \} \geq \lfloor n_k / \eta \rfloor \right\}\\
        &\Rightarrow \left\{ \sum_{i \in S_k \cup C_k} \1\{ r_k < \tau_i \} \geq \lfloor n_k / \eta \rfloor \right\}\\
        &\Rightarrow \left\{ \sum_{i=2}^{n_k+ |C_{k} \backslash (S_{k}\cap C_{k})| + 1} \1\{ r_k < \tau_i \} \geq \lfloor n_k / \eta \rfloor \right\}\\
        &\Leftrightarrow \{ r_k < \tau_{\lfloor n_k / \eta \rfloor +1} \}.
    \end{align*}
    The rest of the proof is the same as that for Successive Halving in \cite{LiHyperband}.
\end{proof}
%
%
\subsection{Comparison of SH and ID-SH} \label{subsec:Proof_theorem_comparison}
\begin{proof}[Proof of Theorem \ref{thm:ImprovementIDSH}]
    Let us first regard the number of total pulls when we run SuccessiveHalving($n, r$) in comparison to a run of xID-SuccessiveHalving($n, r$), where we assume that we had already run SuccessiveHalving($\tilde{n},\tilde{r}$). We concentrate in the following on a lower bound on the pulls of SuccessiveHalving($n,r$).\\
    \begin{align*}
        \sum_{k=0}^{s} n_k r_k &= \sum_{k=0}^{s} \left\lfloor \frac{n}{\eta^k} \right\rfloor \cdot \left\lfloor \frac{R \eta^k}{\eta^s} \right\rfloor \\
        &\geq \sum_{k=0}^{s} \left(\frac{n}{\eta^k} -1 \right) \left(\frac{R\eta^k}{\eta^s}-1 \right) \\
        &= \sum_{k=0}^{s} \frac{nR}{\eta^s} - \frac{R\eta^k}{\eta^s} - \frac{n}{\eta^k} + 1 \\
        &= \frac{(s+1)(nR+\eta^s)}{\eta^s} - \frac{R}{\eta^s} \sum_{k=0}^{s}\eta^k - n \sum_{k=0}^{s} \left(\frac{1}{\eta}\right)^k \\
        &= \frac{(s+1)(nR+\eta^s)}{\eta^s} - \frac{R(\eta^{s+1}-1)}{\eta^s(\eta-1)} - \underbrace{\frac{n(1-(1/\eta)^{s+1})}{1-1/\eta}}_{
        = \frac{n\left(\frac{\eta^{s+1}-1}{\eta^{s+1}}\right)}{\frac{\eta-1}{\eta}} 
        = \frac{n(\eta^{s+1}-1)}{\eta^s(\eta-1)}} \\
        &= \frac{(s+1)(nR+\eta^s)}{\eta^s} - \frac{(\eta^{s+1}-1)(R+n)}{\eta^s(\eta-1)}\\
        &= \frac{(s+1)(nR+\eta^s)(\eta-1) - (\eta^{s+1}-1)(R+n)}{\eta^s(\eta-1)},
    \end{align*}
    where we used the closed form for the geometric series in the fifth line and simple transformations in all other lines.

    \begin{itemize}
        \item[a)] Number of pulls of eID-SH in comparison to SH.\\
        For a comparison we also have to compute an upper bound on the total pulls of eID-SuccessiveHalving($n,r$):
        \begin{align*}
            \sum_{k=0}^{s} n_{k} r_{k} &= \sum_{k=0}^{s} \left( \lfloor n / \eta^{k} \rfloor - \lfloor \tilde{n} / \eta^{k} \rfloor \right) \left\lfloor\frac{R\eta^{k}}{\eta^{s}} \right\rfloor \\
            &\leq \sum_{k=0}^{s} \left( \frac{n-\tilde{n}}{\eta^{k}} + 1 \right) \cdot \frac{R\eta^{k}}{\eta^{s}} \\
            &= \frac{(s+1)(n-\tilde{n})R}{\eta^{s}} + \frac{R}{\eta^{s}} \sum_{k=0}^{s} \eta^{k} \\
            &=  \frac{(s+1)(n-\tilde{n})R}{\eta^{s}} + \frac{R(\eta^{s+1}-1)}{\eta^{s}(\eta-1)}\\
            &=  \frac{(s+1)(n-\tilde{n})R(\eta-1)+R(\eta^{s+1}-1)}{\eta^{s}(\eta-1)}.
        \end{align*}
        To compare both we build the quotient
        \begin{align*}
            \frac{\#\{\text{total pulls of eID-SH($n,r$)}\}}{\#\{\text{total pulls of SH($n,r$)}\}} 
            &\leq \frac{(s+1)(n-\tilde{n})R(\eta-1)+R(\eta^{s+1}-1)}{(s+1)(nR+\eta^s)(\eta-1) - (\eta^{s+1}-1)(R+n)}\\
            &= 1- \frac{(s+1)(\tilde{n}R + \eta^s)(\eta-1) - (\eta^{s+1}-1)(2R+n)}{(s+1)(nR+\eta^s)(\eta-1) - (\eta^{s+1}-1)(R+n)}.
        \end{align*}
        %


%
%


        \color{black}
        \item[b)] Number of pulls of dID-SH in comparison to SH.\\
        Analogue as in a) we first need an upper bound on the total pulls in a run of dID-SuccessiveHalving($n,r$). While we only sample $n-\tilde{n}$ new arms in the first round of dID-SH, the best $n/\eta$ arms may be all from the newly sampled ones and thus none of the arms which are kept into the next round of dID-SH was already pulled with a higher budget in the run of SH($\tilde{n},\tilde{r}$). In this worst case, we can estimate
        \begin{align*}
            \sum_{k=0}^{s}n_k r_k &= (n-\tilde{n})\left\lfloor \frac{R}{\eta^s} \right\rfloor + \sum_{k=1}^{s} \left\lfloor \frac{n}{\eta^k} \right\rfloor \left\lfloor \frac{R\eta^k}{\eta^s} \right\rfloor \\
            &\leq \frac{(n-\tilde{n})R}{\eta^s} + \sum_{k=1}^{s} \frac{nR}{\eta^s}\\
            &= \frac{(n-\tilde{n})R}{\eta^s} + \frac{snR}{\eta^s} \\
            &= \frac{R((s+1)n-\tilde{n})}{\eta^s}.
        \end{align*}
        Again, we can now compute the quotient of the pulls as follows.
        \begin{align*}
            \frac{\#\{\text{total pulls of dID-SH($n,r$)}\}}{\#\{\text{total pulls of SH($n,r$)}\}} 
            &\leq \frac{(\eta-1)R((s+1)n-\tilde{n})}{(\eta-1)(s+1)(nR+\eta^s)-(\eta^{s+1}-1)(R+n)}\\
            &= 1- \frac{(\eta-1)((s+1)\eta^s + R\tilde{n})-(\eta^{s+1}-1)(R+n)}{(\eta-1)(s+1)(nR+\eta^s)-(\eta^{s+1}-1)(R+n)}.
        \end{align*}
    \end{itemize}
    Note that we can apply the same for the number of pulls of pID-SH since we have the same worst-case scenario where we only keep newly sampled configurations into the next round of pID-SH and none of the previously promoted configurations.
\end{proof}
To get an intuition for the improvement in the number of total pulls, we show in Figure \ref{fig:ImprovementcID-SH} and Figure \ref{fig:ImprovementeID-SHeta3} the above terms for different values of rounds $s$, maximal budgets per round $R$ and discarding portion $\eta$. Note that the above results assume the worst-case scenario for the pID-SH resp.\ the dID-SH algorithm in which all previously promoted configurations perform worse than all newly sampled ones. In most problem scenarios the average improvement in the number of total pulls of pID-SH resp.\ dID-SH will lie between the plotted curves of the worst case scenario in Figure \ref{fig:ImprovementcID-SH} and the best case scenario which coincidences with eID-SH and is shown in Figure \ref{fig:ImprovementeID-SHeta3} and \ref{fig:ImprovementeID-SHeta2}. Since our proposed methods will never need a greater number of total pulls than SH, we plotted the minimum value of $1$ and our derived fractions in Theorem \ref{thm:ImprovementIDSH}.

\begin{figure}
    \centering
    \includegraphics[width=0.48\textwidth]{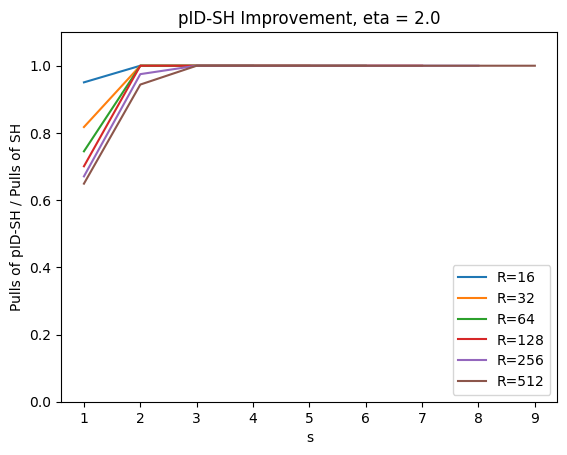}
    \includegraphics[width=0.48\textwidth]{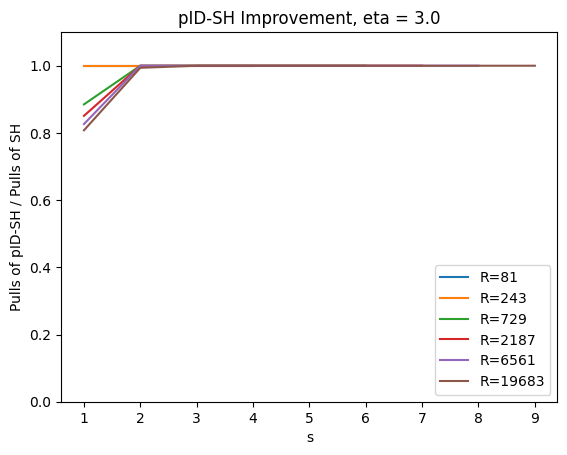}
    \caption{Fraction of the number of total pulls of pID-SH resp. dID-SH and SH for different values of rounds of SH $s$, maximal budgets per round $R$ and discarding fraction $\eta$.}
    \label{fig:ImprovementcID-SH}
\end{figure}
\begin{figure}
    \centering
    \includegraphics[width=0.48\textwidth]{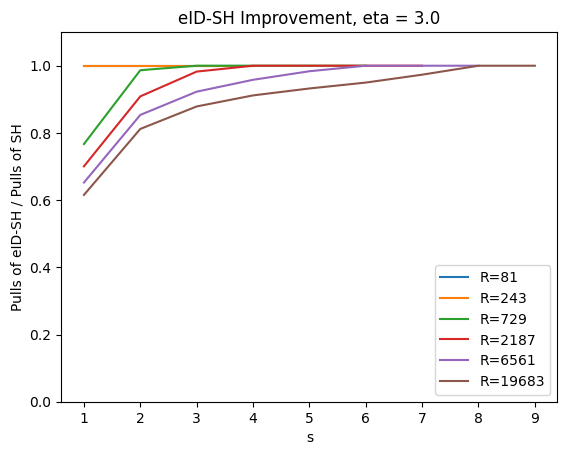}
    \caption{Fraction of number of total pulls of eID-SH for different values of rounds of SH $s$ and maximal budgets per round $R$.}
    \label{fig:ImprovementeID-SHeta3}
\end{figure}
%
%
\newpage \subsection{IterativeDeepening-Hyperband} \label{subsec:theorem_IDHB}
\begin{proof}[Proof of Theorem \ref{thm:Budget_ID-HB}]
    To derive the necessary budget of ID-HB in Algorithm \ref{alg:ID-HB}, we simply have to sum up all necessary budgets for each call of xID-SH. Luckily, the necessary budgets for eID-SH, dID-SH and pID-SH do not differ, thus a run of ID-HB uses independent of the variant of the called Successive Halving algorithm a total budget of
    \begin{align*}
        \sum_{s=0}^{\lfloor \log_{\eta}(R) \rfloor} &\text{Budget\_xID-SH}(n_s, r_s) \\
        &= \sum_{s=0}^{\lfloor \log_{\eta}(R) \rfloor} \eta \left\lceil \log_{\eta}\left( \left\lceil ( \lfloor \log_{\eta}(R) \rfloor + 1) \cdot \frac{\eta^s}{s+1} \right\rceil \right) \right\rceil \cdot \max_{i=2, \dots, n_s} i \left( 1+ \min \left\{ R, \gamma^{-1} \left( \max\left\{ \frac{\epsilon}{4}, \frac{\nu_i- \nu_1}{2} \right\} \right) \right\} \right) \\
        &= (\ast).
    \end{align*}
    Due to simple estimates and transformations, we get
    \begin{align*}
        \eta \left\lceil \log_{\eta}\left( \left\lceil ( \lfloor \log_{\eta}(R) \rfloor + 1) \cdot \frac{\eta^s}{s+1} \right\rceil \right) \right\rceil 
        &\leq \eta \left\lceil \log_{\eta}\left( \left\lceil ( \log_{\eta}(R) + 1) \right\rceil \cdot \left\lceil \frac{\eta^s}{s+1} \right\rceil \right) \right\rceil\\
        &= \eta \left\lceil \log_{\eta}\left( \left\lceil  \log_{\eta}(R) + 1 \right\rceil \right) + \log_{\eta} \left( \left\lceil \frac{\eta^s}{s+1} \right\rceil \right) \right\rceil\\
        &\leq \eta \left\lceil \log_{\eta}\left( \log_{\eta}(R) + 2 \right) + \log_{\eta} \left(  \frac{\eta^s}{s+1} +1 \right) \right\rceil\\
        &\leq \eta \left\lceil \log_{\eta}\left( \log_{\eta}(R) \right) + 2 + \log_{\eta} \left(  \frac{\eta^s}{s+1}\right) + 1 \right\rceil\\
        &= \eta \left\lceil \log_{\eta}\left( \log_{\eta}(R) \right) + \log_{\eta} \left(  \eta^s \right) - \log_{\eta}\left( s+1 \right) + 3 \right\rceil\\
        &\leq \eta \left( \log_{\eta}\left( \log_{\eta}(R) \right) + s - \log_{\eta}\left( s+1 \right) + 4 \right).
    \end{align*}
    Note that the fourth line follows from 
    \begin{align*}
        \log_{\eta}(x+1) &\leq \log_{\eta}(x) + 1 \\
        \Leftrightarrow x+1 &\leq \eta \cdot x \\
        \Leftrightarrow x &\geq \frac{1}{\eta-1}.
    \end{align*}
    In our setting, we have $\eta \geq 2$, thus $\log_{\eta}(x+1) \geq \log_{\eta}(x) + 1$ if and only if $x\geq 1$. We have $\frac{\eta^s}{s+1} \geq 1$ for $s \geq 0$ and also wlog. $\log_{\eta}(R) \geq 2$, otherwise the value of $s_{max}$ and thus the run of Hyperband would be trivial.\\
    We can continue with
    \begin{align*}
        (\ast) &\leq \sum_{s=0}^{\lfloor \log_{\eta}(R) \rfloor} \eta \left( \log_{\eta}\left( \log_{\eta}(R) \right) + s - \log_{\eta}\left( s+1 \right) + 4 \right) \\
        &~~~~~~~\cdot \underbrace{\max_{s = 0,\dots, \lfloor \log_{\eta}(R) \rfloor} \max_{i=2, \dots, n_s} i \left( 1+ \min \left\{ R, \gamma^{-1} \left( \max\left\{ \frac{\epsilon}{4}, \frac{\nu_i- \nu_1}{2} \right\} \right) \right\} \right)}_{=: \bar{\gamma}^{-1}}\\ 
        &= \eta \left( \left( \lfloor \log_{\eta}(R) \rfloor +1 \right) \left( \log_{\eta}(\log_{\eta}(R)) + 4 \right) + \sum_{s=0}^{\lfloor \log_{\eta}(R) \rfloor} s - \sum_{s=0}^{\lfloor \log_{\eta}(R) \rfloor} \log_{\eta}(s+1) \right) \bar{\gamma}^{-1} \\
        &= \eta \left( \left( \lfloor \log_{\eta}(R) \rfloor +1 \right) \left( \log_{\eta}(\log_{\eta}(R)) + 4 \right) + \frac{\lfloor \log_{\eta}(R) \rfloor \left( \lfloor \log_{\eta}(R) \rfloor +1 \right)}{2} - \log_{\eta} \left( \prod_{s=0}^{\lfloor \log_{\eta}(R) \rfloor} (s+1) \right) \right) \bar{\gamma}^{-1} \\
        &= \eta \left( \left( \lfloor \log_{\eta}(R) \rfloor +1 \right) \left( \log_{\eta}(\log_{\eta}(R)) + 4 \right) + \frac{\lfloor \log_{\eta}(R) \rfloor \left( \lfloor \log_{\eta}(R) \rfloor +1 \right)}{2} - \log_{\eta} \left( \left( \lfloor \log_{\eta}(R) \rfloor + 1 \right) ! \right) \right) \bar{\gamma}^{-1}.
    \end{align*}
    Since we choose the budget $B$ in our ID-HB algorithm as $B = (s_{\max} + 1)R = \left(\lfloor \log_{\eta}(R) \rfloor + 1\right)R$, we can divide both by $\left(\lfloor \log_{\eta}(R) \rfloor + 1\right)$ and get
    \begin{align*}
        R \geq \eta \left( \log_{\eta}(\log_{\eta}(R)) + 4 + \frac{\lfloor \log_{\eta}(R) \rfloor}{2} - \frac{\log_{\eta} \left( \left( \lfloor \log_{\eta}(R) \rfloor + 1 \right) ! \right)}{\lfloor \log_{\eta}(R) \rfloor + 1} \right) \bar{\gamma}^{-1}.
    \end{align*}
    %
    %
    Recall that in each call of xID-SH in round $s$ of ID-HB we compare $n_s = (\lfloor\log_{\eta}(R) \rfloor + 1) \frac{\eta^s}{s+1}$ hyperparameter configurations, thus we get an overall number of samples of 
    \begin{align*}
        &\left( \lfloor \log_{\eta}(R) \rfloor + 1 \right) \sum_{s=0}^{\lfloor \log_{\eta}(R) \rfloor} \frac{\eta^s}{s+1} \\
        &\geq \sum_{s=0}^{\lfloor \log_{\eta}(R) \rfloor} \eta^s \\
        \overset{\text{Geometric Sum}}{=} & ~~~ \frac{\eta^{\lfloor \log_{\eta}(R) \rfloor + 1}-1}{\eta - 1} \\
        &\geq \frac{R-1}{\eta-1}.
    \end{align*}
    By assumption \ref{ass:proportion_best_configs} we have an $\epsilon$-optimal hyperparameter configuration in our sample set with probability at least $1-\delta$ if and only if
    \begin{align*}
        & \frac{R -1}{\eta - 1} \geq \left\lceil \log_{1-\alpha}(\delta) \right\rceil \\
        \Leftrightarrow ~~ & R\geq \left\lceil \log_{1-\alpha}(\delta) \right\rceil (\eta -1) + 1.
    \end{align*}
\end{proof}
\color{black}
%
%
\newpage
\section{Detailed Empirical Results}\label{sec:detailed-experimental-results}
In this section, we provide the results of the empirical study in more detail. To this end, we consider the configuration of neural networks in different ways: Tuning the hyperparameters of the learning algorithm for classification problems, tuning the hyperparameters of the learning algorithm for multi-target prediction problems, and for searching a neural architecture.

\begin{figure}[ht]
    \centering
    \begin{minipage}{.98\textwidth}
        \begin{minipage}{.28\textwidth}
            \begin{tikzpicture}
\begin{axis}[title={lcbench, $\eta=3$},xmin=0,xmax=1,ymin=0,ymax=1,legend pos=north west,width=5cm,height=5cm,tick style={grid=major},xlabel=IH-HB,ylabel=Competitors,scatter/classes={eID-HB={mark=triangle,blue},pID-HB={mark=o,red},dID-HB={mark=square,magenta}}]
\addplot[scatter,only marks,scatter src=explicit symbolic] table[meta=label] {
x     y      label
0.8603175333333333  0.8603175333333333  eID-HB
0.8977900666666667  0.8977900666666667  eID-HB
0.8706496  0.8706496  eID-HB
0.9429294666666668  0.9431940333333333  eID-HB
0.9029893333333334  0.9029893333333334  eID-HB
0.8212465000000001  0.8162239333333333  eID-HB
0.7979628000000001  0.7962132666666667  eID-HB
0.6860254  0.6860254  eID-HB
0.9609991333333333  0.9608226999999999  eID-HB
0.634931  0.634931  eID-HB
0.5624206666666667  0.5624206666666667  eID-HB
0.8563342666666666  0.8541968999999999  eID-HB
0.9754073  0.9754073  eID-HB
0.9225532  0.9244604333333333  eID-HB
0.9956356  0.9956356  eID-HB
0.2725406666666667  0.2725406666666667  eID-HB
0.7467784000000001  0.7475293666666667  eID-HB
0.8077118  0.8083157333333334  eID-HB
0.7113761000000001  0.7121127  eID-HB
0.8512844  0.8480894333333333  eID-HB
0.9652293333333334  0.9652271666666666  eID-HB
0.8856574666666667  0.8880609666666667  eID-HB
0.9673518  0.9673518  eID-HB
0.7062479333333334  0.7060526666666667  eID-HB
0.7712847666666666  0.7712847666666666  eID-HB
0.7104052666666667  0.7105400666666667  eID-HB
0.7351291333333333  0.7356307333333334  eID-HB
0.7702097333333333  0.7702097333333333  eID-HB
0.9258165  0.9256037333333332  eID-HB
0.9223799666666667  0.9223799666666667  eID-HB
0.9691689333333333  0.9691405999999999  eID-HB
0.6819801  0.6819801  eID-HB
0.9817732666666666  0.9817732666666666  eID-HB
0.9933332333333333  0.9933332333333333  eID-HB
0.8603175333333333  0.8603175333333333  pID-HB
0.8977900666666667  0.8977900666666667  pID-HB
0.8706496  0.8706496  pID-HB
0.9029893333333334  0.9029893333333334  pID-HB
0.9429294666666668  0.9431940333333333  pID-HB
0.8212465000000001  0.8215897000000001  pID-HB
0.6860254  0.6860254  pID-HB
0.7979628000000001  0.7988822  pID-HB
0.9609991333333333  0.9610484666666667  pID-HB
0.634931  0.634931  pID-HB
0.5624206666666667  0.5624206666666667  pID-HB
0.8563342666666666  0.8563342666666666  pID-HB
0.9754073  0.9754073  pID-HB
0.9225532  0.9244604333333333  pID-HB
0.2725406666666667  0.2725406666666667  pID-HB
0.9956356  0.9956356  pID-HB
0.7467784000000001  0.7475293666666667  pID-HB
0.8077118  0.8083157333333334  pID-HB
0.7113761000000001  0.7113761000000001  pID-HB
0.8512844  0.8533626000000001  pID-HB
0.9652293333333334  0.9652271666666666  pID-HB
0.8856574666666667  0.8868099666666667  pID-HB
0.9673518  0.9673518  pID-HB
0.7062479333333334  0.7062479333333334  pID-HB
0.7712847666666666  0.7712847666666666  pID-HB
0.7104052666666667  0.7104052666666667  pID-HB
0.7702097333333333  0.7702097333333333  pID-HB
0.7351291333333333  0.7356307333333334  pID-HB
0.9258165  0.9258165  pID-HB
0.9223799666666667  0.9223799666666667  pID-HB
0.9691689333333333  0.9691405999999999  pID-HB
0.6819801  0.6819801  pID-HB
0.9817732666666666  0.9817732666666666  pID-HB
0.9933332333333333  0.9933332333333333  pID-HB
0.8603175333333333  0.8603175333333333  dID-HB
0.8977900666666667  0.8977900666666667  dID-HB
0.8706496  0.8706496  dID-HB
0.9029893333333334  0.9029893333333334  dID-HB
0.9429294666666668  0.9429294666666668  dID-HB
0.8212465000000001  0.8212465000000001  dID-HB
0.7979628000000001  0.7979628000000001  dID-HB
0.6860254  0.6860254  dID-HB
0.9609991333333333  0.9609991333333333  dID-HB
0.634931  0.634931  dID-HB
0.5624206666666667  0.5624206666666667  dID-HB
0.8563342666666666  0.8563342666666666  dID-HB
0.9754073  0.9754073  dID-HB
0.9225532  0.9225532  dID-HB
0.2725406666666667  0.2725406666666667  dID-HB
0.9956356  0.9956356  dID-HB
0.7467784000000001  0.7467784000000001  dID-HB
0.8077118  0.8077118  dID-HB
0.7113761000000001  0.7113761000000001  dID-HB
0.8512844  0.8512844  dID-HB
0.9652293333333334  0.9652293333333334  dID-HB
0.8856574666666667  0.8856574666666667  dID-HB
0.9673518  0.9673518  dID-HB
0.7062479333333334  0.7062479333333334  dID-HB
0.7712847666666666  0.7712847666666666  dID-HB
0.7104052666666667  0.7104052666666667  dID-HB
0.7702097333333333  0.7702097333333333  dID-HB
0.7351291333333333  0.7351291333333333  dID-HB
0.9258165  0.9258165  dID-HB
0.9223799666666667  0.9223799666666667  dID-HB
0.9691689333333333  0.9691689333333333  dID-HB
0.6819801  0.6819801  dID-HB
0.9817732666666666  0.9817732666666666  dID-HB
0.9933332333333333  0.9933332333333333  dID-HB
};\addlegendentry{eID-HB}
\addlegendentry{pID-HB}
\addlegendentry{dID-HB}
\addplot[color=black] coordinates {
	(0,0)
	(1,1)
};
\end{axis}
\end{tikzpicture}
        \end{minipage}
        \begin{minipage}{.21\textwidth}
            \input{scatter-plots-rbv2-ranger-3.tex}
        \end{minipage}
        \begin{minipage}{.21\textwidth}
            \input{scatter-plots-rbv2-svm-3.tex}
        \end{minipage}
        \begin{minipage}{.21\textwidth}
            \input{scatter-plots-rbv2-xgboost-3.tex}
        \end{minipage}
        \vspace{-.8cm}
    \end{minipage}
    \begin{minipage}{.98\textwidth}
        \begin{minipage}{.28\textwidth}
            \hfill
            \includegraphics[width=.75\textwidth]{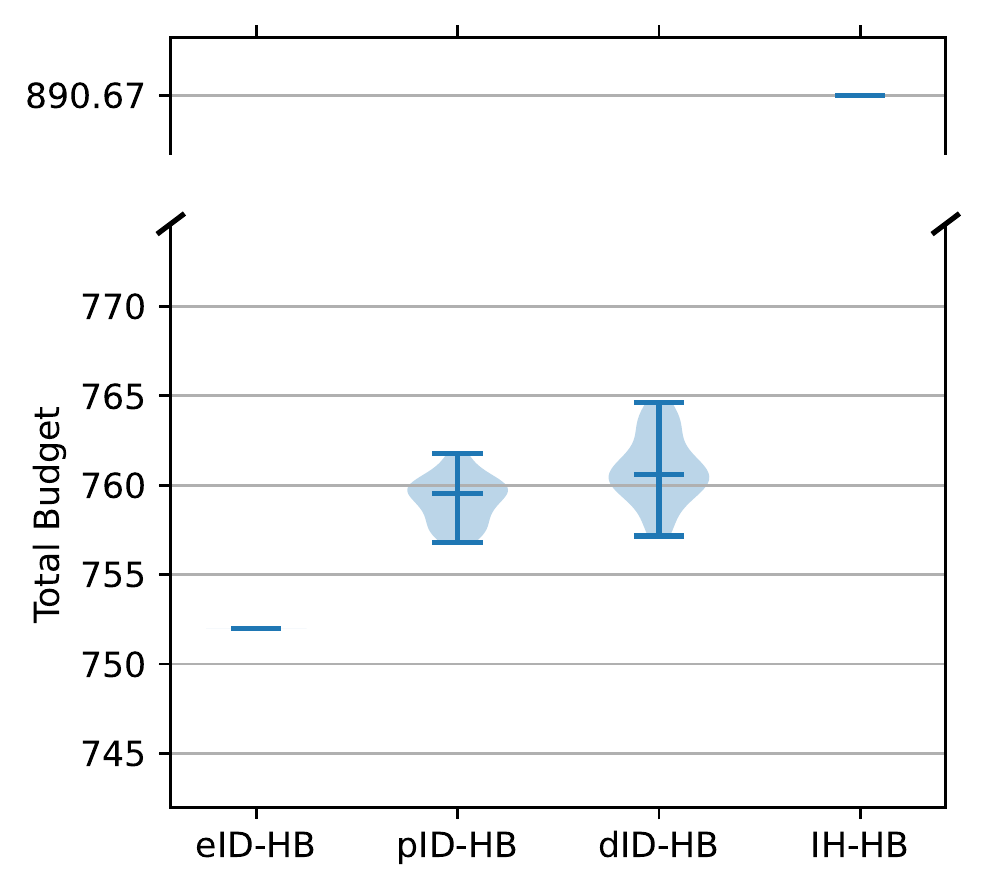}
        \end{minipage}
        \begin{minipage}{.21\textwidth}
            \includegraphics[width=\textwidth]{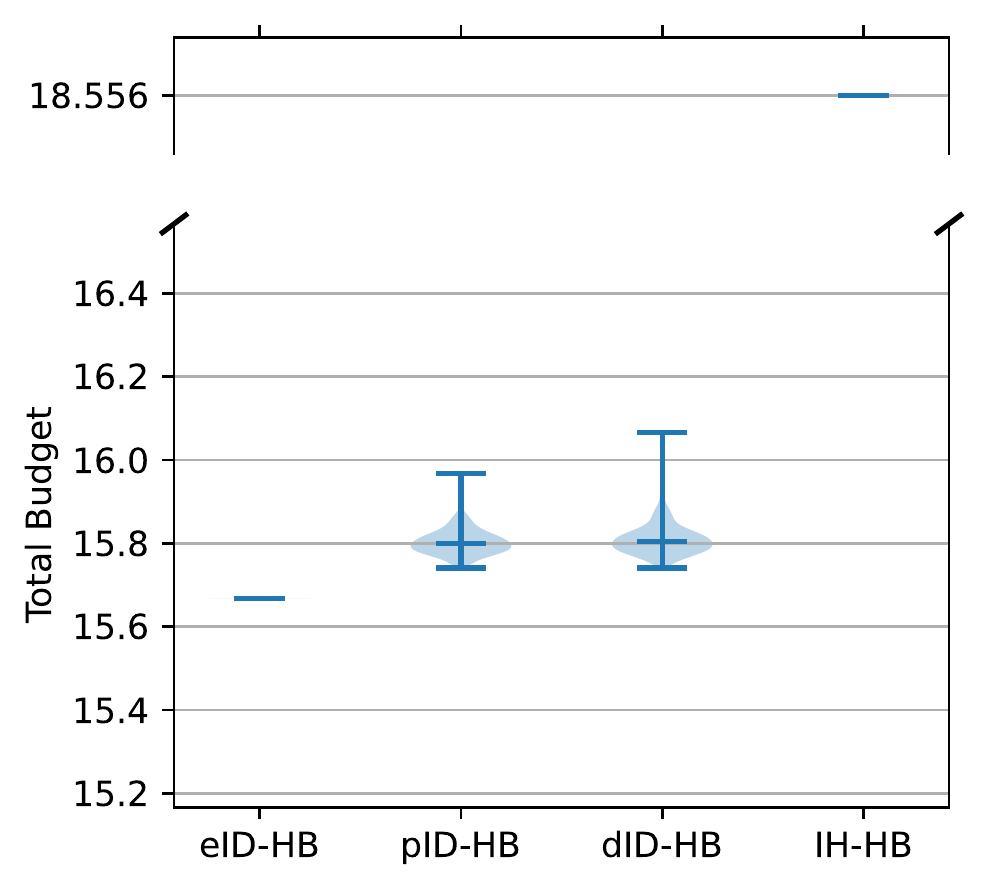}
        \end{minipage}
        \begin{minipage}{.21\textwidth}
            \includegraphics[width=\textwidth]{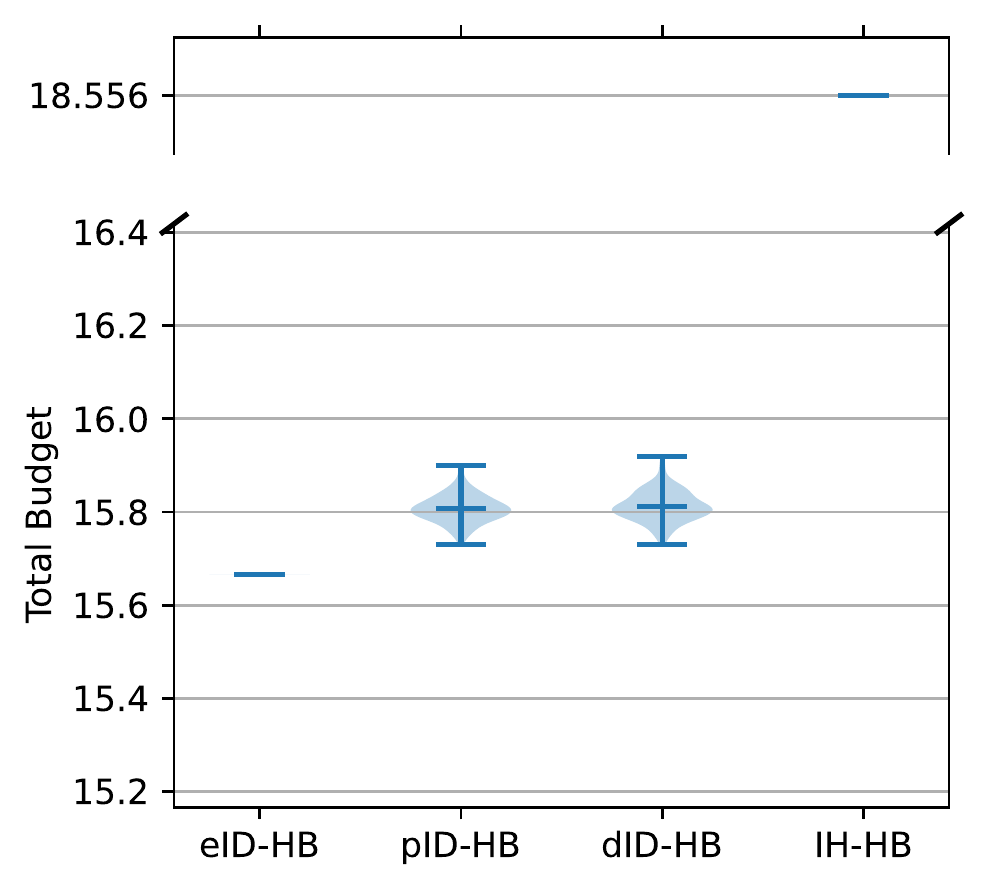}
        \end{minipage}
        \begin{minipage}{.21\textwidth}
            \includegraphics[width=\textwidth]{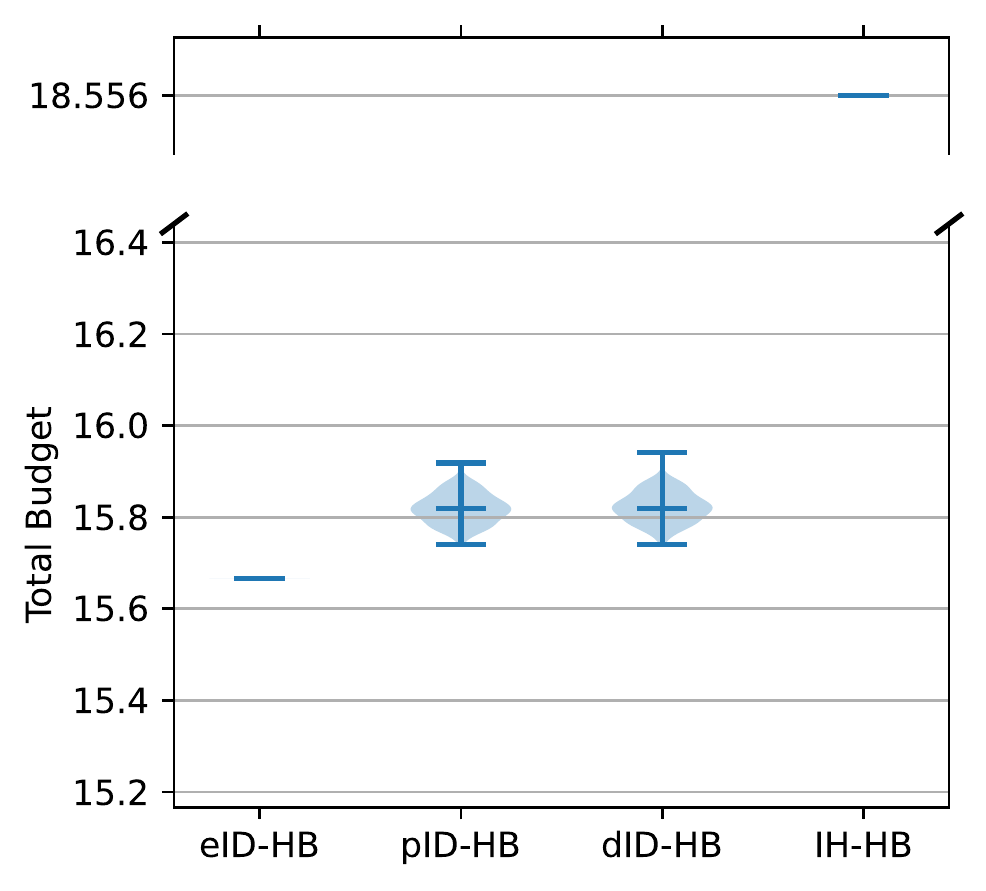}
        \end{minipage}
    \end{minipage}
    \caption{Comparison of ID-HB to the original version of \hb. Top: Scatter plots plotting the final incumbents' accuracy obtained via an ID-HB strategy versus IH-HB. Bottom: Violin plots showing the average total budget consumed for a single run.}
    \label{fig:result-plots-appendix}
\end{figure}

\begin{table}[h]
\centering
\begin{tabular}{lllll}

\toprule
\multicolumn{5}{c}{nb301, $\eta=2$}\\
\midrule

OpenML ID & IH-HB & \discardingidhb & \conservativehb & \efficienthb \\
\midrule

CIFAR10 & .906$\pm$.00 & .906$\pm$.00 & .906$\pm$.00 & .906$\pm$.00\\
\bottomrule
\end{tabular}
\end{table}
\begin{table}[h]
\centering
\begin{tabular}{lllll}

\toprule
\multicolumn{5}{c}{nb301, $\eta=3$}\\
\midrule

OpenML ID & IH-HB & \discardingidhb & \conservativehb & \efficienthb \\
\midrule

CIFAR10 & .939$\pm$.00 & .939$\pm$.00 & .939$\pm$.00 & .939$\pm$.00\\

\bottomrule
\end{tabular}
\end{table}
\begin{table}[h]
\centering
\begin{tabular}{lllll}
\toprule
\multicolumn{5}{c}{lcbench, $\eta=2$}\\
\midrule
OpenML ID & IH-HB & \discardingidhb & \conservativehb & \efficienthb \\
\midrule
3945 & .993$\pm$.00 & .993$\pm$.00 & .993$\pm$.00 & .993$\pm$.00\\
7593 & .716$\pm$.03 & .716$\pm$.03 & .716$\pm$.03 & .716$\pm$.03\\
34539 & .981$\pm$.01 & .981$\pm$.01 & .981$\pm$.01 & .981$\pm$.01\\
126025 & .881$\pm$.02 & .881$\pm$.02 & .879$\pm$.02 & .877$\pm$.02\\
126026 & .966$\pm$.01 & .966$\pm$.01 & .966$\pm$.01 & .965$\pm$.01\\
126029 & .926$\pm$.01 & .926$\pm$.01 & .926$\pm$.01 & .926$\pm$.01\\
\bottomrule
\end{tabular}
\end{table}
\begin{table}[h]
\centering
\begin{tabular}{lllll}
\toprule
\multicolumn{5}{c}{lcbench, $\eta=3$}\\
\midrule

OpenML ID & IH-HB & \discardingidhb & \conservativehb & \efficienthb \\
\midrule
3945 & .993$\pm$.00 & .993$\pm$.00 & .993$\pm$.00 & .993$\pm$.00\\
7593 & .711$\pm$.04 & .711$\pm$.04 & .711$\pm$.04 & .712$\pm$.04\\
34539 & .982$\pm$.01 & .982$\pm$.01 & .982$\pm$.01 & .982$\pm$.01\\
126025 & .886$\pm$.02 & .886$\pm$.02 & .887$\pm$.02 & .888$\pm$.02\\
126026 & .965$\pm$.01 & .965$\pm$.01 & .965$\pm$.01 & .965$\pm$.01\\
126029 & .923$\pm$.02 & .923$\pm$.02 & .924$\pm$.02 & .924$\pm$.02\\
\bottomrule
\end{tabular}
\end{table}
\centering
\begin{longtable}{lllll}
\toprule
\multicolumn{5}{c}{rbv2\_ranger, $\eta=2$}\\
\midrule

OpenML ID & IH-HB & dID-HB & pID-HB & eID-HB \\
\midrule
3 & .995$\pm$.00 & .995$\pm$.00 & .995$\pm$.00 & .995$\pm$.00\\
6 & .950$\pm$.01 & .950$\pm$.01 & .950$\pm$.01 & .950$\pm$.01\\
11 & .899$\pm$.01 & .899$\pm$.01 & .899$\pm$.01 & .899$\pm$.01\\
12 & .972$\pm$.00 & .972$\pm$.00 & .972$\pm$.00 & .972$\pm$.00\\
14 & .836$\pm$.01 & .836$\pm$.01 & .836$\pm$.01 & .836$\pm$.01\\
15 & .978$\pm$.01 & .978$\pm$.01 & .978$\pm$.01 & .978$\pm$.01\\
\bottomrule
\end{longtable}
\begin{longtable}{lllll}
\toprule
\multicolumn{5}{c}{rbv2\_ranger, $\eta=3$}\\
\midrule

OpenML ID & IH-HB & \discardingidhb & \conservativehb & \efficienthb \\
\midrule

3 & .994$\pm$.00 & .994$\pm$.00 & .994$\pm$.00 & .994$\pm$.00\\
6 & .946$\pm$.01 & .946$\pm$.01 & .946$\pm$.01 & .946$\pm$.01\\
11 & .897$\pm$.01 & .897$\pm$.01 & .897$\pm$.01 & .897$\pm$.01\\
12 & .971$\pm$.00 & .971$\pm$.00 & .971$\pm$.00 & .971$\pm$.00\\
14 & .833$\pm$.01 & .833$\pm$.01 & .833$\pm$.01 & .833$\pm$.01\\
15 & .977$\pm$.01 & .977$\pm$.01 & .977$\pm$.01 & .977$\pm$.01\\
\bottomrule
\end{longtable}
\begin{longtable}{lllll}

\toprule
\multicolumn{5}{c}{rbv2\_svm, $\eta=2$}\\
\midrule

OpenML ID & IH-HB & \discardingidhb & \conservativehb & \efficienthb \\
\midrule

3 & .994$\pm$.00 & .994$\pm$.00 & .994$\pm$.00 & .994$\pm$.00\\
6 & .971$\pm$.01 & .971$\pm$.01 & .971$\pm$.01 & .971$\pm$.01\\
11 & .969$\pm$.02 & .969$\pm$.02 & .969$\pm$.02 & .969$\pm$.02\\
12 & .981$\pm$.00 & .981$\pm$.00 & .981$\pm$.00 & .981$\pm$.00\\
14 & .845$\pm$.01 & .846$\pm$.01 & .846$\pm$.01 & .846$\pm$.01\\
15 & .973$\pm$.01 & .973$\pm$.01 & .973$\pm$.01 & .973$\pm$.01\\
\bottomrule
\end{longtable}
\begin{longtable}{lllll}

\toprule
\multicolumn{5}{c}{rbv2\_svm, $\eta=3$}\\
\midrule

OpenML ID & IH-HB & \discardingidhb & \conservativehb & \efficienthb \\
\midrule

3 & .993$\pm$.00 & .993$\pm$.00 & .993$\pm$.00 & .993$\pm$.00\\
6 & .968$\pm$.01 & .968$\pm$.01 & .968$\pm$.01 & .968$\pm$.01\\
11 & .957$\pm$.03 & .957$\pm$.03 & .957$\pm$.03 & .957$\pm$.03\\
12 & .980$\pm$.00 & .980$\pm$.00 & .980$\pm$.00 & .980$\pm$.00\\
14 & .844$\pm$.01 & .844$\pm$.01 & .844$\pm$.01 & .844$\pm$.01\\
15 & .972$\pm$.01 & .972$\pm$.01 & .972$\pm$.01 & .972$\pm$.01\\
\bottomrule
\end{longtable}
\begin{longtable}{lllll}

\toprule
\multicolumn{5}{c}{rbv2\_xgboost, $\eta=2$}\\
\midrule

OpenML ID & IH-HB & \discardingidhb & \conservativehb & \efficienthb \\
\midrule

3 & .985$\pm$.01 & .985$\pm$.01 & .985$\pm$.01 & .985$\pm$.01\\
6 & .000$\pm$.00 & .000$\pm$.00 & .000$\pm$.00 & .000$\pm$.00\\
11 & .857$\pm$.02 & .857$\pm$.02 & .857$\pm$.02 & .857$\pm$.02\\
12 & .983$\pm$.00 & .983$\pm$.00 & .983$\pm$.00 & .983$\pm$.00\\
14 & .852$\pm$.02 & .852$\pm$.02 & .852$\pm$.02 & .852$\pm$.02\\
15 & .984$\pm$.01 & .984$\pm$.01 & .984$\pm$.01 & .984$\pm$.01\\
\bottomrule
\end{longtable}
\begin{longtable}{lllll}
\toprule
\multicolumn{5}{c}{rbv2\_xgboost, $\eta=3$}\\
\midrule
OpenML ID & IH-HB & \discardingidhb & \conservativehb & \efficienthb \\
\midrule

3 & .979$\pm$.01 & .979$\pm$.01 & .979$\pm$.01 & .979$\pm$.01\\
6 & .999$\pm$.01 & .999$\pm$.01 & .999$\pm$.01 & .999$\pm$.01\\
11 & .845$\pm$.02 & .845$\pm$.02 & .845$\pm$.02 & .845$\pm$.02\\
12 & .982$\pm$.00 & .982$\pm$.00 & .982$\pm$.00 & .982$\pm$.00\\
14 & .848$\pm$.02 & .848$\pm$.02 & .848$\pm$.02 & .848$\pm$.02\\
15 & .982$\pm$.01 & .982$\pm$.01 & .982$\pm$.01 & .982$\pm$.01\\
\bottomrule
\end{longtable}

\end{document}